%% file: camera_ready.tex
\documentclass{article}

\usepackage{microtype}
\usepackage{graphicx}
\usepackage{subfigure}
\usepackage{booktabs} % for professional tables
\usepackage[utf8]{inputenc} % allow utf-8 input
\usepackage[T1]{fontenc}    % use 8-bit T1 fonts
                     \usepackage{hyperref}       % hyperlinks
\usepackage{url}            % simple URL typesetting
\usepackage{booktabs}       % professional-quality tables
\usepackage{amsfonts}       % blackboard math symbols
\usepackage{nicefrac}       % compact symbols for 1/2, etc.
\usepackage{microtype}      % microtypography
\usepackage[math]{alp}
\usepackage{amsthm}
\usepackage{amsmath}
\usepackage{amssymb}
\usepackage{multirow}
\usepackage{enumitem}
\usepackage{paralist}
\usepackage{placeins}
\usepackage{float}
\usepackage{bbm}
\usepackage{hyperref}
\hypersetup{colorlinks,allcolors=black}

\usepackage[colorinlistoftodos,
           prependcaption,
           textsize=small,
           backgroundcolor=yellow,
           linecolor=lightgray,
           bordercolor=lightgray]{todonotes}

\DeclareRobustCommand{\parhead}[1]{\textbf{#1}~}

\newtheorem{theorem}{Theorem}
\newtheorem{lemma}{Lemma}

\usepackage{hyperref}

\usepackage[nameinlink]{cleveref}
\creflabelformat{equation}{#1#2#3}
\crefname{equation}{eq.}{eqs.}  
\Crefname{equation}{Eq.}{Eqs.}

\usepackage[accepted]{icml2021}
\hypersetup{linkcolor=black}

\icmltitlerunning{Order Matters: Probabilistic Modeling of Node Sequence for Graph Generation}

\begin{document}

\twocolumn[
\icmltitle{Order Matters: Probabilistic Modeling of Node Sequence\\for Graph Generation}



\icmlsetsymbol{equal}{*}

\begin{icmlauthorlist}
\icmlauthor{Xiaohui Chen}{equal,to1}
\icmlauthor{Xu Han}{equal,to1}
\icmlauthor{Jiajing Hu}{to1}
\icmlauthor{Francisco J.\ R.\ Ruiz}{to2}
\icmlauthor{Liping Liu}{to1}
\end{icmlauthorlist}

\icmlaffiliation{to1}{Department of Computer Science, Tufts University, Medford, MA, USA.}

\icmlaffiliation{to2}{DeepMind, London, UK}
% \icmlaffiliation{goo}{Googol ShallowMind, New London, Michigan, USA}
% \icmlaffiliation{ed}{School of Computation, University of Edenborrow, Edenborrow, United Kingdom}

\icmlcorrespondingauthor{Xiaohui Chen}{xiaohui.chen@tufts.edu}
% \icmlcorrespondingauthor{Xu Han}{xu.han@tufts.edu}

\icmlkeywords{Machine Learning, ICML}

\vskip 0.3in
]



\printAffiliationsAndNotice{\icmlEqualContribution} % otherwise use the standard text.

\begin{abstract}

\input{sections/abstract}
\end{abstract}

\input{sections/introduction}

\input{sections/distribution_of_a_sequential_generation_model}

\input{sections/optimizing_generation_order}
\input{sections/experiments}

\input{sections/conclusion}
\input{sections/acknowledgments}

% \nocite{10.1093/nar/gky1033, adams1995hitchhiker, albert2002statistical, anonymous, bacciu2019graph, cai2016edge, caron2017sparse, cordella2004sub, corneil1970efficient, diaconis2007graph, DudaHart2nd, hamilton2017inductive, irwin2005zinc, kearns89, kingman1978representation, kipf2016semi, klicpera2018predict, langley00, liu2018constrained, MachineLearningI, maddison2014sampling, mitchell80, Newell81, newman2018networks, popova2019molecularrnn, Samuel59, schomburg2004brenda, Sen_Namata_Bilgic_Getoor_Galligher_Eliassi-Rad_2008, simonovsky2018graphvae, tucker2017rebar, veitch2015class, velivckovic2017graph}
\bibliography{references.bib}
\bibliographystyle{icml2021}
\clearpage
\appendix
\input{sections/Appendix}
\end{document}

%% file: sections/abstract.tex
A graph generative model defines a distribution over graphs. One type of generative model is constructed by autoregressive neural networks, which sequentially add nodes and edges to generate a graph. However, the likelihood of a graph under the autoregressive model is intractable, as there are numerous sequences leading to the given graph; this makes maximum likelihood estimation challenging. Instead, in this work we derive the exact joint probability over the graph and the node ordering of the sequential process. From the joint, we approximately marginalize out the node orderings and compute a lower bound on the log-likelihood using variational inference. We train graph generative models by maximizing this bound, without using the ad-hoc node orderings of previous methods. Our experiments show that the log-likelihood bound is significantly tighter than the bound of previous schemes. Moreover, the models fitted with the proposed algorithm can generate high-quality graphs that match the structures of target graphs not seen during training. We have made our code publicly available at \hyperref[https://github.com/tufts-ml/graph-generation-vi]{https://github.com/tufts-ml/graph-generation-vi}.

%% file: sections/introduction.tex
\section{Introduction}
% graph generation background
% previous method
% setbacks of previous method
% proposed method & result and performance

Random graphs have been a prominent topic in statistics and graph theory for decades. An early and influential model of random graphs is the Erd\H{o}s–R\'enyi model \citep{erdHos1960evolution}. Since then, various models have been proposed to characterize different global statistics of graphs or networks in the real world \citep{watts1998collective,nowicki2001estimation,cai2016edge}. However, these models are usually not designed for capturing local structures of a graph, such as bonds in a molecule graph.

Autoregressive generative models \citep{you2018graphrnn, li2018learning, liao2019efficient, dai2020scalable, goyal2020graphgen, yuan2020xgnn,shi2020graphaf} are designed to learn fine structures in graph data. These models generate a graph by sequentially adding nodes and edges. Since a graph is invariant to node permutations \citep{veitch2015class}, there are multiple sequences of actions leading to the same graph. When fitting an autoregressive model to data, a particular node ordering $\pi$ of the graph $G$ (called ``generation order'') is used to pin down a single generation sequence of $G$, such as depth-first search (DFS) or breadth-first search (BFS) ordering. The model is then fitted assuming the graph was generated under such ordering $\pi$. Autoregressive models of graphs typically use deep learning tools \citep{guo2020systematic}, such as recurrent neural networks (RNNs), to learn flexible and complex patterns from data.

Choosing a specific ordering $\pi$ does not rigorously correspond to maximum likelihood estimation (MLE). Indeed, to fit the parameters of an autoregressive model via MLE, we need the likelihood of $G$ under the model. One approach for computing $p(G)$ is to sum over all possible node orderings $\pi$, $p(G)=\sum_{\pi} p(G, \pi)$. However, this approach presents some challenges. First, a generation sequence of $G$ corresponds to multiple node orderings when $G$ has non-trivial automorphisms \citep{you2018graphrnn,liao2019efficient}, which require us to carefully derive the joint $p(G, \pi)$ from the model's distribution of generation sequences. Second, the marginalization is intractable in practice due to the number of terms in the sum. As a consequence, $p(G)$ cannot be easily obtained. This does not only make MLE intractable, but also implies that generative models cannot be evaluated in terms of log-likelihood. Instead, other evaluation metrics such as degree distribution are used, but these metrics exhibit some issues for complex graphs \citep{liu2019auto}.

% One approach for computing the log-likelihood of $G$ under these autoregressive models is to sum over all possible node orderings $\pi$. However, a generation sequence of $G$ corresponds to multiple node orderings when $G$ has non-trivial automorphisms \citep{you2018graphrnn,liao2019efficient}. Before mariginalization of $\pi$, an exact form of $p(G, \pi)$ needs to be derived from an autoregressive model's distribution of generation sequences. Moreover, the marginalization of $\pi$ from the joint $p(G,\pi)$ is intractable in practice. As a consequence, the marginal log-likelihood $\log p(G)$ cannot be easily obtained; this makes maximum likelihood estimation (MLE) intractable. In addition, the generative models cannot be evaluated in terms of log-likelihood, and instead some other metrics such as degree distribution are used, but these metrics exhibit some issues for complex graphs \citep{liu2019auto}.\looseness=-1
%\todo{this citation requires the year}

In this work, we provide a method to estimate the marginal log-likelihood, enabling standard statistical model checking and comparison. It also opens the door for other learning tasks that require the log-likelihood of graph data, such as density-based anomaly detection.

We aim at consolidating the foundation of autoregressive graph generative models. In particular, we examine two types of models: one that generates a graph through an evolving graph sequence and one that generates an adjacency matrix. Then we derive the joint $p(G, \pi)$ from each type. Our analysis reveals a relationship between graph generation and graph automorphism.

% When fitting large graphs via MLE, the marginalization of $\pi$ in $p(G, \pi)$ is hard. To avoid that, we perform
To fit large graphs via MLE, we avoid the intractable marginalization by performing
approximate posterior inference over the node ordering $\pi$. In particular, we use variational inference (VI) and maximize a lower bound of $\log p(G)$. We design a neural network that infers the probability over $\pi$ for a given graph $G$. Thus, the generative model is trained with node orderings that are likely to generate $G$, avoiding the need to define ad-hoc orderings.

For evaluation, we estimate the graph log-likelihood via importance sampling. Our empirical study indicates that the variational lower bound is relatively tight. We also find that generative models fitted with the proposed method perform better than existing methods according to various metrics, including log-likelihood. Models trained with our method are able to generate new graphs with higher similarity to training graphs than existing approaches.

\parhead{Contributions.}
Our main contributions are as follows:
\begin{compactitem}
    \item we give a rigorous definition of the probability of  node orderings in autoregressive graph generative models;
    \item we analyze the relation between the calculation of graph probabilities and graph automorphism;
    \item we introduce VI to infer node orderings; and 
    \item our training method with VI improves the performance of the model both quantitatively and qualitatively.
\end{compactitem}

\parhead{Related work.}
Autoregressive graph generation models have gained attention due to both the quality of generated graphs and their generation efficiency \citep{you2018graphrnn,li2018learning,liao2019efficient,dai2020scalable,shi2020graphaf}.
In these works, $\pi$ is often decided by DFS or BFS, or it can be a specially designed canonical order. \citet{liao2019efficient} justify this approach by showing that these methods optimize a variational bound on $\log p(G)$. However, when the node orderings $\pi$ are either randomly sampled from a uniform distribution or limited to a small range of canonical orders, these bounds are likely to be loose.

One model that considers a single canonical node ordering $\pi^\star$ is GraphGEN \citep{goyal2020graphgen}. That is, for a given graph $G$, GraphGEN obtains $p(G)$ by considering that the graph was generated according to $\pi^\star$. However, when generating a graph from the model, GraphGEN does not guarantee the canonical order. This design raises a theoretical issue: the frequency of a generation sequence may not converge to the model's probability of that sequence.

%% file: sections/distribution_of_a_sequential_generation_model.tex
\section{Autoregressive Graph Generation}
\label{sec:autoregressive_graph_generation}

In \Cref{subsec:problem}, we introduce the two formulations of an autoregressive generative model---based on either a graph sequence or an adjacency matrix. In \Cref{subsec:generation_order}, we provide an explicit relationship between each formulation and the node ordering $\pi$ to obtain the exact joint $p(G,\pi)$. %In \Cref{subsec:issues_canonical_order}, we discuss the issues of defining a single canonical node ordering.

\begin{figure*}[t]
    \begin{center}
    \includegraphics[width=0.98\textwidth]{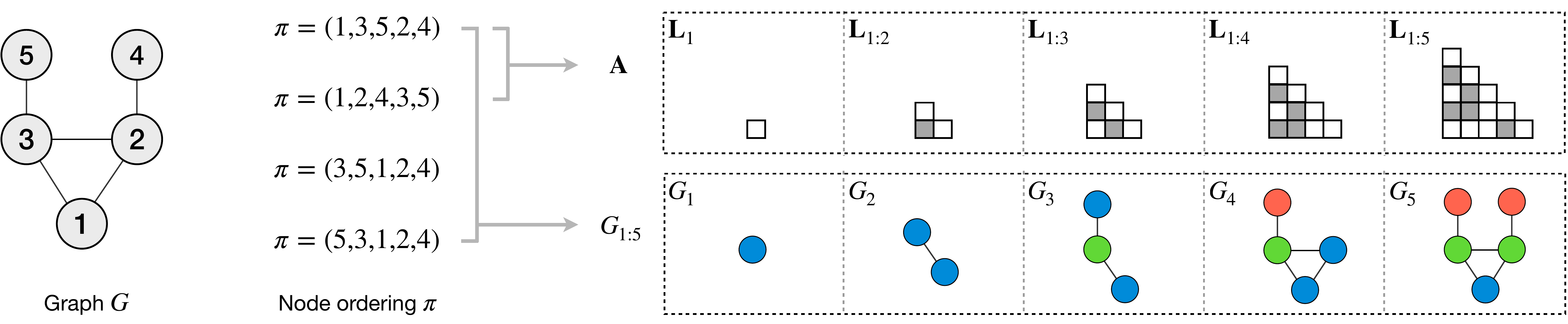}
    \end{center}
    \caption{An overview of the relationship between the node ordering $\pi$ and the adjacency matrix $\bA$ or the graph sequence $G_{1:n}$. Given a graph $G$ (left), several node orderings $\pi$ (middle) specify the same adjacency matrix $\bA$ or graph sequence $G_{1:n}$, so we cannot uniquely identify $\pi$ from either $\bA$ or $G_{1:n}$. The node orderings that give the same $\bA$ give the same $G_{1:n}$, but not vice-versa. In the plot of $G_{1:n}$, for each subgraph $G_t$, nodes in the same orbit are labeled with the same color.}
    \label{fig:framework}
\end{figure*}

\subsection{Problem definition}
\label{subsec:problem}

Let $V = \{1, \ldots, n\}$ and $E$ be the node set and edge set of a graph $G = (V, E)$ with $|V| = n$ nodes. A node ordering $\pi = (\pi_1, \ldots, \pi_n)$ is a permutation of the elements in $V$. We consider $G$ is unlabeled: permuting the nodes does not change the graph. The graph has a class $\calA(G)$ of adjacency matrices corresponding to different node orderings---for each $\pi$, there is a unique adjacency matrix $\bA \in \{0, 1\}^{n \times n}$ that indicates which nodes are connected. We only consider finite graphs without self-loops and multi-edges, so $\bA$ is symmetric and its diagonal elements are zero. Let $\calG$ denote the space of such graphs.

A generative model of unlabeled graphs defines a distribution $p(G)$ over $\calG$. The distribution must be invariant to permutation of graph nodes. In this work, we focus on autoregressive generative models. We next review two formulations of autoregressive generative models.

The autoregressive model\footnote{%
The formulation by \citet{liao2019efficient} generates graph nodes in batches, but it can also be expressed as an autoregressive model in this form. Similarly, GraphGEN \citep{goyal2020graphgen}, which generates the sparse form of each row of $\bL$, is also in this form.}
by \citet{you2018graphrnn,liao2019efficient,shi2020graphaf,goyal2020graphgen} operates with the adjacency matrix $\bA$. In particular, the model generates a lower triangular matrix $\bL$ by sequentially generating each row of $\bL$. After every row is generated, it may stop with a special termination symbol, denoted by $\otimes$. Since an adjacency matrix $\bA = \bL + \bL^\top$, each $\bL$ uniquely determines $\bA$ and vice-versa; thus $p(\bA) = p(\bL)$ and 
\begin{align}
    \label{eq:graph-gen}
    p(\bA)= p(\otimes | \bL) \prod_{t = 2}^{n} p(L_{t,:} | \bL_{1:(t - 1)}).
\end{align}
Here, $\bL_{1:(t - 1)}$ denotes the submatrix formed from the first $t-1$ rows of $\bL$, and $L_{t,:}$ is the $t$-th row of $\bL$. The probability $p(L_{1, :})=1$ is left out here. The adjacency matrix $\bA$ fully defines a graph $G$. 

The deep generative model of graphs (DeepGMG) \citep{li2018learning} defines the sequential process as follows. It starts with a graph $G_1$ with one node, and at each step $t=2,3,\ldots$, it obtains a graph $G_t$ by adding a new node as well as some edges connecting the new node to the previously generated graph $G_{t-1}$. The probability of the sequence $G_{1:n} = (G_1, \ldots, G_n)$ is
\begin{align}
  p(G_{1:n})  = p( \otimes | G_n)  \prod_{t=2}^{n} p(G_t | G_{t - 1}).
\end{align}
The probability $p(G_1)=1$ is left out here as well. Note that, after $n$ steps, the graph $G_n$ is the generated graph $G$.

A given graph $G$ does not naturally have either a unique adjacency matrix $\bA$ or a unique graph sequence $G_{1:n}$. Therefore, when fitting these models, we need to specify a node ordering $\pi$ to pin down a single adjacency matrix $\bA$ or sequence $G_{1:n}$. We depart from these two formulations and consider a formal treatment of the node ordering.%\looseness=-1

\subsection{The generation order as a random variable}
\label{subsec:generation_order}

Here, we relate the sequential processes from \Cref{subsec:problem} with the node ordering $\pi$. First, we consider the marginal likelihood $p(G)$. Under the first formulation, we obtain $p(G)$ by marginalizing over all adjacency matrices of $G$,
\begin{align}
    \label{eq:marg-adj}
    p(G) = \sum_{\bA \in \calA(G)} p(\bA).
\end{align}
Under the second formulation, the marginalization is over all graph sequences that lead to $G$, i.e.,
\begin{align}
    \label{eq:marg-seq}
    p(G) = \sum_{G_{1:n}:~ G_n = G} p(G_{1:n}).
\end{align}

In both cases, the likelihood is intractable because the marginalization space is hard to specify---it involves finding all unique adjacency matrices or graph sequences \citep{liao2019efficient}. To obtain $p(G)$, many works use instead the node ordering $\pi$ as the marginalization variable since the space of $\pi$ is easier to characterize than that of $\bA$ or $G_{1:n}$ for a graph $G$. To obtain $p(G,\pi)$, we need to clarify the relationship between $\bA$ or $G_{1:n}$ and $\pi$, as we discuss next.

% start talking about node orderings
The sequential process from \Cref{subsec:generation_order} generates an adjacency matrix or graph sequence; however in general we cannot identify $\pi$ from either of these variables. To see this, consider first the relation between $\bA$ and the node ordering $\pi$. Given the graph $G$, $\pi$ determines $\bA$ because the $t$-th row of $\bA$ correponds to node $\pi_t$. However, the converse is not necessarily true: a matrix $\bA$ corresponds to multiple node orderings if $G$ has non-trivial automorphism \citep{liao2019efficient}. We provide an example in \Cref{fig:framework}, where each of the first two node orderings ($\pi=(1,3,5,2,4)$ and $\pi=(1,2,4,3,5)$) determines $\bA$, but we cannot uniquely identify one of them from $\bA$ (in particular, we cannot distinguish the node pairs $(2,4)$ and $(3,5)$). The same is true for the graph sequence $G_{1:n}$: a node ordering $\pi$ defines a graph sequence $G_{1:n}$, but not vice-versa (see \Cref{fig:framework}).

Similarly, the relation between $\bA$ and $G_{1:n}$ is not unique. An adjacency matrix $\bA$ determines a graph sequence $G_{1:n}$, but a graph sequence does not determine a unique $\bA$. As an example, in \Cref{fig:framework} all four node orderings generate the same $G_{1:n}$, but the last two node orderings determine two adjacency matrices different from the shown matrix $\bA$.

In summary, $(G, \pi)$ determines $\bA$, which determines $G_{1:n}$, but the reverse is not true in general. This implies that \emph{an autoregressive generative model (which generates $\bA$ or $G_{1:n}$) does not specify a distribution over $\pi$}.

We next make $\pi$ a random variable and formally specify the joint $p(G, \pi)$. %Let $\bA = \gamma_a(G, \pi)$ and $G_{1:n} = \gamma_s(G, \pi)$ denote the surjections for obtaining $\bA$ and $G_{1:n}$ from $(G, \pi)$, respectively; i.e., the functions that map from $(G,\pi)$ to either the adjacency matrix $\bA$ or the graph sequence $G_{1:n}$.  
Given the graph $G = (V, E)$, let $\Pi[\bA]$ be the set of all possible node orderings $\pi$ that give the same adjacency $\bA$; similarly, let $\Pi[G_{1:n}]$ be the set of all node orderings that give the same graph sequence $G_{1:n}$, i.e., 
\begin{align}
\Pi[\bA] &= \{\pi: A_{\pi_i, \pi_j} = \mathbbm{1}[(\pi_i, \pi_j) \in E], \forall i, j \in V \} \nonumber \\
\Pi[G_{1:n}] &= \{\pi: G[\pi_{1:t}] = G_{t}, \forall t = 1, \ldots, n\}. \nonumber 
\end{align}
Here, $\mathbbm{1}[\cdot]$ is $1$ or $0$ depending on whether the condition in the bracket is true or false, and $G[\pi_{1:t}]$ is the induced subgraph of $G$ from the first $t$ nodes in the ordering $\pi$. 
Then we let the conditional distribution $p(\pi | \bA)$ be uniform, i.e.,
\begin{align}
    p(\pi | \bA) = \frac{1}{\big| ~\Pi[\bA] ~\big|}.
\end{align}
The set $\Pi[\bA]$ turns out to be the set of automorphisms\footnote{%
A function  $f:V \rightarrow V$ is an automorphism of $G=(V,E)$ if $(u, v) \in E \Longleftrightarrow (f(u), f(v)) \in E$.}
of the graph $G$. This is because every node ordering $\pi \in \Pi[\bA]$ permutes rows and columns of $\bA$ but does not change $\bA$; that is, each $\pi$ creates an automorphism. Therefore, obtaining $p(\pi | \bA)$ amounts to finding the number of automorphisms of a graph. Fortunately, this is a well-studied classic problem in graph theory. The time complexity of computing $|\Pi[\bA]|$ is $\exp{\left(\mathcal{O}(\sqrt{n \log n})\right)}$ \citep{beals1999finding}. The Nauty package \citep{mckay2013nauty} uses various heuristics and can efficiently find this number for most graphs. In practice, it can compute $|\Pi[\bA]|$ for a graph with thousands of nodes within $10^{-3}$ seconds.

For the formulation with graph sequences, the analysis is more involved. We define the conditional $p(\pi | G_{1:n})$ as a uniform distribution,
\begin{align}
    p(\pi | G_{1:n}) = \frac{1}{\big|~\Pi[G_{1:n}]~\big|}.
\end{align}
We discuss below how to obtain $|\Pi[G_{1:n}]|$ in practice, but first we formally specify the joint $p(G, \pi)$ and the likelihood $p(G)$. The joint can be obtained from $p(\bA)$ or $p(G_{1:n})$ as
\begin{align}
    \label{eq:joint_p_G_pi}
    p(G, \pi) =  \frac{1}{\big|~\Pi[\bA]~\big|} p(\bA) = \frac{1}{\big|~\Pi[G_{1:n}]~\big|} p(G_{1:n}).
\end{align}
(This expression assumes that $\bA\in\calA(G)$ and that $G_n = G$.) The marginal likelihood $p(G)$ of a graph can be obtained by marginalizing out the node ordering $\pi$ from \Cref{eq:joint_p_G_pi},
\begin{align}
    \label{objective marginal}
    p(G) =\sum_{\pi} p(G, \pi).
\end{align}
Obtaining $p(G)$ from \Cref{objective marginal} is easier than from \Cref{eq:marg-adj} or \Cref{eq:marg-seq} because the marginalization space is easier to characterize, but it remains intractable because of the large number of terms in the sum. In \Cref{subsec:VI}, we derive a variational bound on $p(G)$ by approximating the posterior distribution $p(\pi|G)$, for which we use the definition of the joint in \Cref{eq:joint_p_G_pi}.

\parhead{Obtaining $|\Pi[G_{1:n}]|$.}
We now discuss the practical calculation of $|\Pi[G_{1:n}]|$. Like $|\Pi[\bA]|$, it is also closely related to graph automorphism. Let $\mathrm{Aut}(G)$ denote the set of all automorphisms of $G$, then the \textit{orbit} of a node $u\in V$ is $r(G, u) = \{v\in V: \exists f \in \mathrm{Aut}(G), v = f(u)\}$ \citep{godsil2001algebraic}. Intuitively, the orbit of $u$ contains all nodes that are ``symmetric'' to $u$. In \Cref{fig:framework}, the orbit of node $3$ is $\{2, 3\}$, and the orbit of node $5$ is $\{4, 5\}$. The theorem below expresses $|\Pi[G_{1:n}]|$ in terms of the cardinality of the orbits produced during the sequential generative process.

\begin{theorem} For a graph sequence $G_{1:n}$, we have
    \begin{align}
        \label{eq:repetition}
        \big|~\Pi[G_{1:n}]~\big| = \prod_{t=1}^n ~ \big|~r(G_t, \pi_t)~\big|.
    \end{align}
\end{theorem}
We show an example before providing the proof. Suppose that $G_n$ is the complete graph with $n$ nodes, then each $G_t$ in the sequence is a complete graph with $t$ nodes. Applying the theorem with $r(G_t, \pi_t) = t$ gives $|\Pi[G_{1:n}]| = n!$, which means that all $n!$ permutations use the same graph sequence.
\begin{proof}
    The proof of the theorem needs the following lemma, whose proof is in \Cref{app:proof_lemma1}.
    \begin{lemma}
        \label{lma1}
        Let $G[V\backslash\{u\}]$ and  $G[V\backslash\{v\}]$ respectively denote the subgraphs induced by $V\backslash\{u\}$ and   $V\backslash\{v\}$, then $u$ and $v$ are in the same orbit if and only if $G[V\backslash\{u\}]$ and $G[V\backslash\{v\}]$ are isomorphic.
    \end{lemma}
    We prove \Cref{eq:repetition} by induction. Let $\pi \in \Pi[G_{1:n}]$, and consider the number of node orderings that give the same graph sequence as $\pi$. When $n = 1$, there is only one node in the graph, and then the base case is true: $|\Pi[G_{1}]| = |r(G_1, 1)| = 1$. Then, we show the induction rule $|\Pi[G_{1:n}]| = |\Pi[G_{1:(n-1)}]| \cdot |r(G_n, \pi_n)|$. If a node ordering $\pi'$ of $G_n$ gives the same graph sequence $G_{1:n}$ as $\pi$, then nodes $\pi'_n$ and $\pi_n$ must be in the same orbit by the lemma. There are $|r(G_n, \pi_n)|$ choices of $\pi'_n$. Then, consider the number of choices for $\pi'_{1:(n-1)}$. Since removing $\pi_n$ and removing  $\pi'_n$ give two isomorphic graphs, $\pi'_{1:(n-1)}$ can take any node ordering in $\Pi[G_{1:(n-1)}]$ and thus has $|\Pi[G_{1:(n-1)}]|$ possible values. Together, $\pi'$ has $|\Pi[G_{1:n}]|$ possible values, which implies the induction rule. 
\end{proof}

To compute $r(G_t, \pi_t)$, we need to identify the orbit of the node $\pi_t$, which can be expensive for some graphs. Thus, we resort instead to an approximation of $r(G_t, \pi_t)$ that ultimately results in a lower bound of $p(G)$. The approximation is based on the color refinement algorithm (1-Weisfeiler-Lehman), which approximately obtains the orbit of a node. The algorithm uses node colors to partition nodes and always assigns the same color to nodes in the same orbit \citep{arvind2017graph}. Let $\bc = \mathrm{CR}(G)$ be node colors from the color refinement algorithm; then $r_{\mathrm{CR}}(G, u) = \{v \in V: c_v=c_u\} \supseteq r(G, u)$ (the two sets are equal for most cases since the color refinement algorithm is very effective in practice). Then, we can use the result of the algorithm to obtain a bound of \Cref{eq:repetition},
\begin{align}
    \beta(G_{1:n}) \triangleq  \prod_{t=1}^n |r_{\mathrm{CR}}(G_t, \pi_t)| \geq \big|~\Pi[G_{1:n}] ~\big|. 
    \label{eq:beta}
\end{align}
This implies a bound on the joint $p(G, \pi)$ from \Cref{eq:joint_p_G_pi},
\begin{align}
    \widehat{p}(G,\pi) \triangleq  \frac{1}{\beta(G_{1:n})} p(G_{1:n}) \leq p(G,\pi).
\label{eq:approx_pgpi} 
\end{align}
This bound is tight in practice because of the effectiveness of the color refinement algorithm. In \Cref{subsec:VI}, we optimize a variational bound on the marginal $\sum_{\pi} \widehat{p}(G,\pi)\approx p(G)$, but we write $p(G)$ and $p(G,\pi)$ for simplicity.

\parhead{Can we avoid the marginalization by using a single generation order for a graph?} GraphGEN \citep{goyal2020graphgen} defines a single canonical node ordering $\pi^\star$ for a given graph $G$. Then, there is only one adjacency matrix $\bA^\star$ corresponding to $\pi^\star$, and GraphGEN defines $p(G)=p(\bA^\star)$, therefore avoiding the marginalization over $\pi$. However, GraphGEN does not restrict the generation order when sampling from the model; in fact there is not a straightforward way to control the generation order because the canonical order is computed retrospectively after $G$ is generated. As a result, a sample from GraphGEN may be generated with a node ordering that is different from the canonical order of the resulting graph. Thus, the sampling probability of $G$ is likely to be inconsistent with the probability $p(G)$ that the model assigns to $G$. That is, the sampling frequency of $G$ will not converge to the model's $p(G)$, which is a severe problem for a statistical model. To estimate how different the sampling and the model probabilities are, we tested the generation procedure of GraphGEN, and we found that only 9.1\% of the generated graphs use the canonical order that is used for the calculation of $p(G)$ during training.

%% file: sections/optimizing_generation_order.tex
\section{Training a Generative Model using VI}
\label{subsec:VI}

Here we present a method to fit an autoregressive graph generation model that does not rely on any constraints on the node ordering. We use the notation $p_{\theta}(G, \pi)$ to explicitly indicate that the joint depends on the parameters $\theta$ of the generative model---either $p_{\theta}(\bA)$ or $p_{\theta}(G_{1:n})$. For moderately large graphs, the MLE of $\theta$ is computationally intractable because the marginalization of $\pi$ from \Cref{objective marginal} involves $n!$ terms; we sidestep this issue with a VI method \citep{blei2017variational} that maximizes a lower bound on $\log p_{\theta}(G)$.

The variational lower bound $L(\theta,\phi, G) \leq \log p_{\theta}(G)$ is  
\begin{align}\label{eq:elbo}
    L(\theta,\phi, G) \!=\! \E{q_{\phi}(\pi | G)}{\log p_{\theta}(G, \pi) \!-\! \log q_{\phi}(\pi | G)}\!.
\end{align}
Here $q_{\phi}(\pi|G)$ is a variational distribution to approximate the posterior $p_{\theta}(\pi|G)$. Its parameters  are denoted by $\phi$. 
We fit the model parameters $\theta$ and the variational parameters $\phi$ by maximizing \Cref{eq:elbo} w.r.t.\ both parameters. 
We discuss the form of the variational distribution $q_{\phi}(\pi|G)$ in \Cref{subsec:variational_distribution} and the optimization algorithm in  \Cref{subsec:maximizing_elbo}. %Finally we discuss the log-likelihood model evaluation using importance sampling in \Cref{subsec:evaluation_is} . 

\begin{figure}[t]
    \begin{center}
    \includegraphics[width=0.48\textwidth]{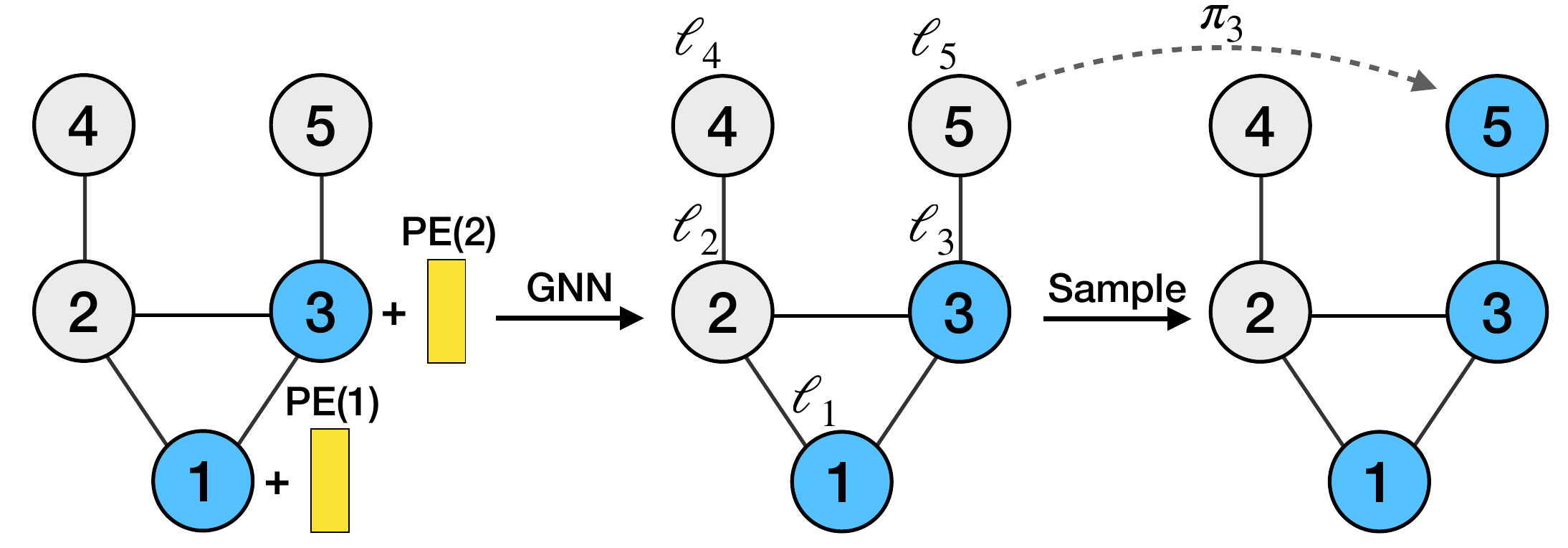}
    \end{center}
    \vspace*{-8pt}
    \caption{Illustration of the sampling procedure from the variational distribution. (Left) To sample node $\pi_3=5$ given $\pi_1=1$ and $\pi_2=3$, we first augment the initial node features with positional embeddings using \Cref{eq:p-e}. (Middle)
    The GNN obtains the logits for each node using \Cref{eq:logits}. (Right) We sample $\pi_3$ from the categorical distribution (\Cref{eq:q-t}).}
    \label{fig:vi}
\end{figure}

\subsection{The variational distribution}
\label{subsec:variational_distribution}

% The true posterior $p(\pi|G;\theta)$ for the node ordering during the graph generation process is intractable. To best approximate it, $q(\pi|G;\phi)$ needs a form to incorporate both graph topological information as well as partially generated order. We come up with a Recurrent Ordering Structure (ROS) to specify $q(\pi|G;\phi)$. Particularly, 
% \begin{align}
% q(\pi|G;\phi) = \prod_{i=1}^n q(\pi_i|\pi_{1:i-1},G;\phi)
% \end{align}
% For each  step, the $i$-th node $\pi_i$ is conditioned on both the graph structure and the partially generated order $\pi_{1:i-1}$. 

The variational distribution $q_{\phi}(\pi|G)$ approximates the intractable posterior $p_{\theta}(\pi|G)$. To obtain a good approximation, we let $q_{\phi}(\pi|G)$ incorporate both graph topological information as well as the information from partially generated graphs according to the order $\pi$. We use a Recurrent Ordering Structure (ROS) to specify $q_{\phi}(\pi|G)$,
\begin{align}
    q_{\phi}(\pi|G) = \prod_{t=1}^n q_{\phi}(\pi_t|G,\pi_{1:(t-1)}).
\end{align}
At each step, the distribution of the $t$-th node $\pi_t$ depends on both $G$ and the partial order $\pi_{1:(t-1)}$. In particular, the conditional $q_{\phi}(\pi_t|G,\pi_{1:(t-1)})$ is a categorical distribution over $\pi_t$; we denote its logits by $\{\ell_{k}^t\}$, then
\begin{align}
    \label{eq:q-t}
    q_{\phi}(\pi_t| G,\pi_{1:(t-1)}) = \frac{\exp\left\{\ell_{\pi_t}^t\right\}}{\sum_{k \notin \pi_{1:(t-1)}}\!\!\! \exp\left\{\ell_{k}^t\right\}}, ~\pi_t\notin \pi_{1:(t-1)}.
\end{align}
The logits are functions of $(G,\pi_{1:(t-1)})$. We use a graph neural network (GNN) as the recurrent unit that outputs the logits $\{\ell_{k}^t\}$ of the conditional $q_{\phi}(\pi_t|G,\pi_{1:(t-1)})$, since GNNs are powerful tools to extract information from graphs. The input of a GNN usually consists of the graph $G$ and its node features; in our case the input is  $\pi_{1:(t-1)}$ and $G$. To encode $\pi_{1:(t-1)}$ into an initial set of node features $\{\bh_1^t,...,\bh_n^t\}$, we use a positional embedding $\mathrm{PE}(\cdot)$ \citep{vaswani2017attention}, such that
\begin{align}
    \label{eq:p-e}
    \bh_{j}^t = 
    \begin{cases}
    \bh_0 + \mathrm{PE}(t), &  \mbox{if  $j = \pi_{t'}$ for $t' < t$,} \\
    \bh_0, &  \mbox{otherwise.}
    \end{cases}
\end{align}
Here, $\bh_0$ is a learnable vector used globally for all steps and nodes. (If the graph data contains node features, we can use these node features to replace $\bh_0$.)
Then, the GNN computes the logits  for all nodes. 
\begin{align}
    \label{eq:logits}
    (\ell_1^t, \ldots, \ell_n^t) = \mathrm{GNN}_{\phi}(G, (\bh_1^t, \ldots, \bh_n^t)).
\end{align}

\begin{algorithm}[t]
  \caption{VI algorithm for training a graph model based on the adjacency matrix $\bA$ }
  \label{alg:training_algorithm_graphrnn}
  \begin{algorithmic}
    \STATE {\bfseries Input:} Dataset of graphs $\calG = \{G_1,\ldots,G_n\}$, 
    model $p_{\theta}$, variational distribution $q_{\phi}$, sample size $S$ %, transition function $\gamma_{a}(\cdot)$
    \STATE {\bfseries Output:} Learned parameters $\theta$ and $\phi$
    \REPEAT
      \FOR{$G \in \calG$}
        \STATE Sample $\pi^{(1)},\ldots, \pi^{(S)}  \overset{\textrm{iid}}{\sim} q_{\phi}(\pi|G)$
        \STATE Obtain $\bA^{(s)}$ from $(G, \pi^{(s)})$
        \STATE Set $p_{\theta}(G,\pi^{(s)}) = \frac{1}{|\Pi[\bA^{(s)}]|} p_{\theta}(\bA^{(s)})$
        \STATE Compute $\nabla_{\phi} \leftarrow \nabla_{\phi} L(\theta,\phi,G)$
        \STATE Compute $\nabla_{\theta} \leftarrow \nabla_{\theta} L(\theta,\phi,G)$
        \STATE Update $\phi$, $\theta$ using the gradients $\nabla_{\phi}$, $\nabla_{\theta}$
      \ENDFOR
    \UNTIL{convergence of the parameters ($\theta$, $\phi$)}
  \end{algorithmic}
\end{algorithm}
Only logits for nodes not in $\pi_{1:(t-1)}$ are used for the calculation of \eqref{eq:q-t}. 
\Cref{fig:vi} illustrates the process to sample from the conditional $q_{\phi}(\pi_t|G, \pi_{1:(t-1)})$.

The choice of the specific GNN is flexible. In our experiments, the graph attention network (GAT) \citep{velivckovic2017graph} performed better than the graph convolutional network (GCN) \citep{wu2019simplifying} and the approximate personalized propagation of neural predictions (APPNP) \citep{klicpera2018predict}. All results in \Cref{subsec:exp} use the GAT.

\subsection{Maximizing the variational lower bound}
\label{subsec:maximizing_elbo}

To maximize the lower bound $L(\theta,\phi,G)$ in \Cref{eq:elbo}, we need its gradients w.r.t.\ both $\theta$ and $\phi$, which are intractable. We obtain the gradient w.r.t.\ $\theta$ via Monte Carlo estimation. We obtain the gradient w.r.t.\ $\phi$ using the score function estimator \citep{williams1992simple,Carbonetto2009,Paisley2012,Ranganath2014}.
The estimators are obtained with $S$ samples $\pi^{(s)}\sim q_{\phi}(\pi|G)$ for $s=1,\ldots,S$, yielding
\begin{align}
    \nabla_{\theta} L(\theta,\phi,G) & \approx \frac{1}{S} \sum_{s=1}^S  \nabla_{\theta} \log p_{\theta}(G, \pi^{(s)}),
    \label{eq:theta-grad}\\
% \end{align}
% \begin{align}
    \nabla_{\phi} L(\theta,\phi,G) & \approx \frac{1}{S}  \sum_{s=1}^S \Big[\log p_{\theta}(G, \pi^{(s)})  \label{eq:phi-grad} \\
    & \hspace{0.4cm} - \log q_{\phi}(\pi^{(s)} | G)\Big] \nabla_\phi \log q_{\phi}(\pi^{(s)} | G). \nonumber
\end{align}
\Cref{eq:theta-grad} shows that the parameters $\theta$ of the model are optimized under node sequences $\pi$ sampled from the approximate posterior. That is, fitting the model does not require to define ad-hoc orderings $\pi$; rather, the (approximately) most likely node orderings are used. As a comparison, a model trained with uniformly distributed random node orderings can be seen as using a uniform variational distribution, which in turn corresponds to a looser log-likelihood bound.

Although the score function estimator may exhibit large variance in general, in our experiments we found that this does not represent an issue.
In fact, $S=4$ samples were enough and allowed for stable optimization of the objective (see \Cref{app:variance}). We leave other gradient estimation techniques \citep{mohamed2019monte} for future work.

We present the training procedure in \Cref{alg:training_algorithm_graphrnn}. The algorithm can be applied to many autoregressive models operating with the adjacency matrix $\bA$, such as GraphRNN and GraphGEN. For models that operate with the graph sequence instead, such as DeepGMG, we only need to extract the graph sequence $G_{1:n}^{(s)}$ from each $(G,\pi^{(s)})$ and set $p_{\theta}(G,\pi^{(s)}) = \frac{1}{|\Pi[G_{1:n}^{(s)}]|} p_{\theta}(G_{1:n}^{(s)})$.

\parhead{Running time.}
To form the gradient estimators, each of the $S$ Monte Carlo samples requires $n$ evaluations of the GNN output, each taking $\mathcal{O}(|E|)$. For most graphs, the complexity of the gradient computation is dominated by these terms and is therefore $\mathcal{O}(Sn|E|)$. Counting automorphisms only takes a small fraction of the running time in practice. Similarly, the approximation of $|\Pi[G_{1:n}]|$ also takes a small fraction of the running time.
The resulting $\mathcal{O}(Sn|E|)$ complexity is a limitation of the proposed algorithm, and hence it is hard to scale to large graphs. However, since it provides better results than existing approaches (see \Cref{subsec:exp}), our algorithm can still be preferable for applications that are not sensitive to the training time. We leave for future work the exploration of ways to improve the computational efficiency, such as proposing the node ordering $\pi$ in one shot.

\begin{table*}[t]
    \small
    \centering
    \caption{Approximate test log-likelihood and variational lower bound (ELBO) of different graph generation models. For each model, we compare the default training algorithm with our method based on VI; the table shows that VI improves the model's predictive performance. Moreover, the variational bound is relatively tight. We used paired $t$-test to compare the results; the numbers in bold indicate that the method is better at the $5\%$ significance level.}
    % \caption{Approximate log-likelihoods of different generation models}. We performed paired t-test to compare results, and the numbers in bold are the ones significantly better on the $5\%$ significance level.
    \label{tab:marginal-likelihood}
    \scalebox{0.95}{
    \input{tables/marginal_and_elbo}
    }
    \vspace*{-4pt}
\end{table*}

%% file: tables/marginal_and_elbo.tex
\begin{tabular}{cc c c c c c c}
    \hline
        %  && \multicolumn{2}{c}{Community-small} &   \multicolumn{2}{c}{Citeseer-small}  &  \multicolumn{2}{c}{Enzymes}&   \multicolumn{2}{c}{Lung}&  \multicolumn{2}{c}{Yeast}& \multicolumn{2}{c}{Cora}\\
        
         && Community-small &   Citeseer-small  &  Enzymes &   Lung &  Yeast & Cora\\
        && log-like/ELBO & log-like/ELBO & log-like/ELBO & log-like/ELBO & log-like/ELBO& log-like/ELBO\\
        \hline
        \multirow{2}{*}{DeepGMG}   
        & uniform & -206.2/-303.9 & \textbf{-60.9/-67}& -281.9/-290.8&\textbf{-146.7/-225.7}&-115.1/128.9&-283.7/-295.2\\
        & VI [ours] &\textbf{-124.8/-131.8}& \textbf{-59.6/-65.6}&\textbf{-145.8/-156.2}& \textbf{-146.1/-224.6}&\textbf{-105.4/-115.7}&\textbf{-227/-247.2}\\
        \hline
        \multirow{2}{*}{GraphRNN}
        & uniform & -154.6/-157.6 & -101.9/-105.7 & -340.3/-349.1 & -232.4/ -242.2 & -189.3/-200.1 &-380.6/-401.8\\
        & VI [ours] & \textbf{-53.7/-59.9} & \textbf{-89.6/-93.2} & \textbf{-274.9/-282.8} & \textbf{-155.9/-175.8} & \textbf{-109.1/-133.7} & \textbf{-345.3/-358.3}\\
    \hline
        \multirow{2}{*}{GraphGEN}
        & DFS &-263.74/NA &-73.0/NA & -574.2/NA &-140.1/NA & \textbf{-66.46/NA} & -199.5/NA\\
        & VI [ours] & \textbf{-26.6/-35.0} & \textbf{-64.3/-71.1} & \textbf{-189.7/-213.8}& \textbf{-117.3/-125.5}  &-\textbf{64.98/-72.39} & \textbf{-143.6/-152.3}\\
        \hline
    % \hline \hline
    %     DeepGMG & -27.99 & \textbf{-11} & &  \\
    %     GraphRNN & -143.53& \textbf{-56.30} & &\\
    %     GraphGEN & & & & \\
    % \hline
    \end{tabular}

%% file: sections/experiments.tex
\begin{figure*}[t]
    \centering
    \begin{tabular}{cc|cc|cc|cc}
    \multicolumn{2}{c}{Ground Truth} & \multicolumn{2}{c}{BFS} & \multicolumn{2}{c}{uniform} & \multicolumn{2}{c}{VI [ours]}\\
    \includegraphics[width=0.09\textwidth]{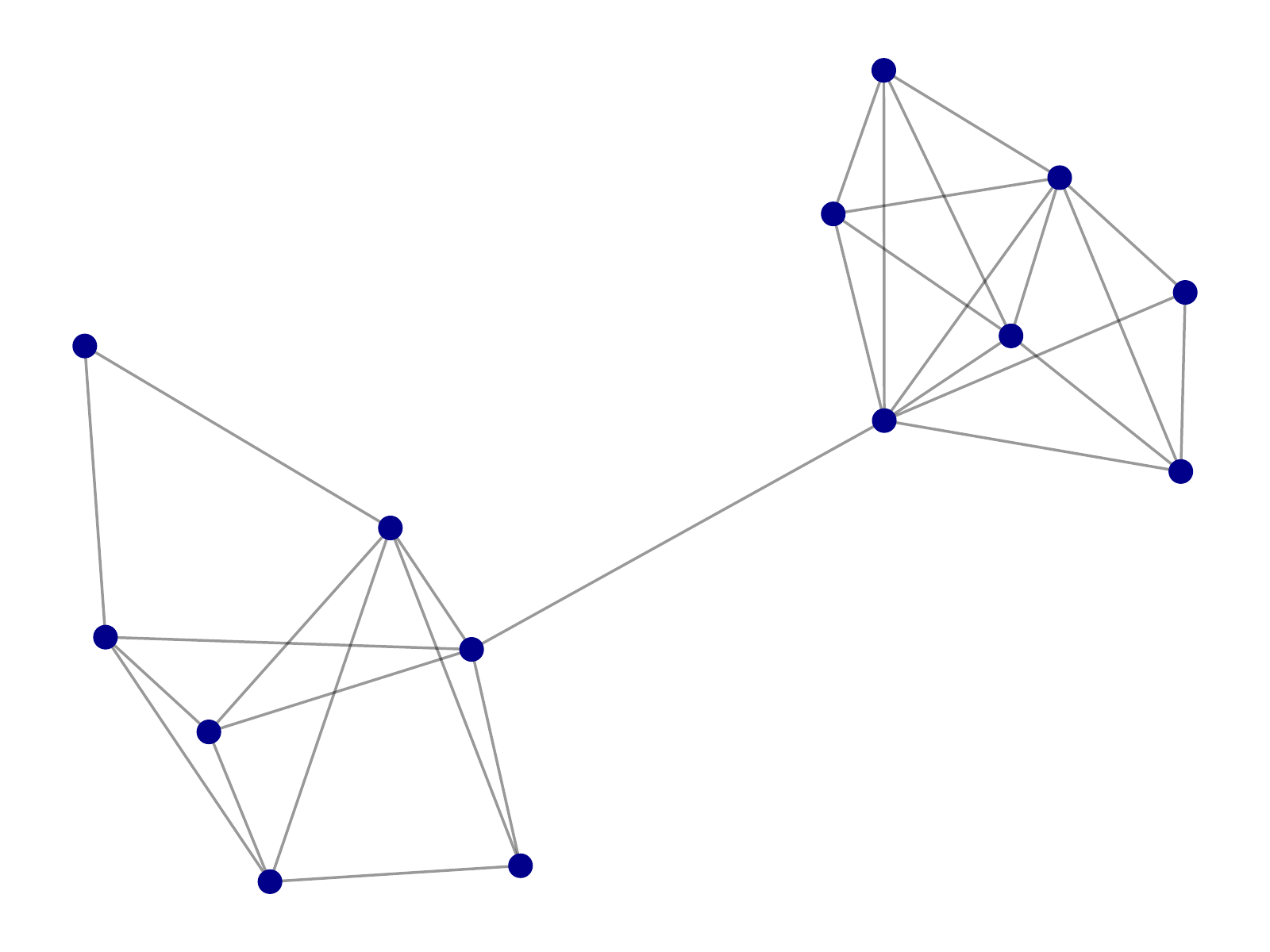}&\includegraphics[width=0.09\textwidth]{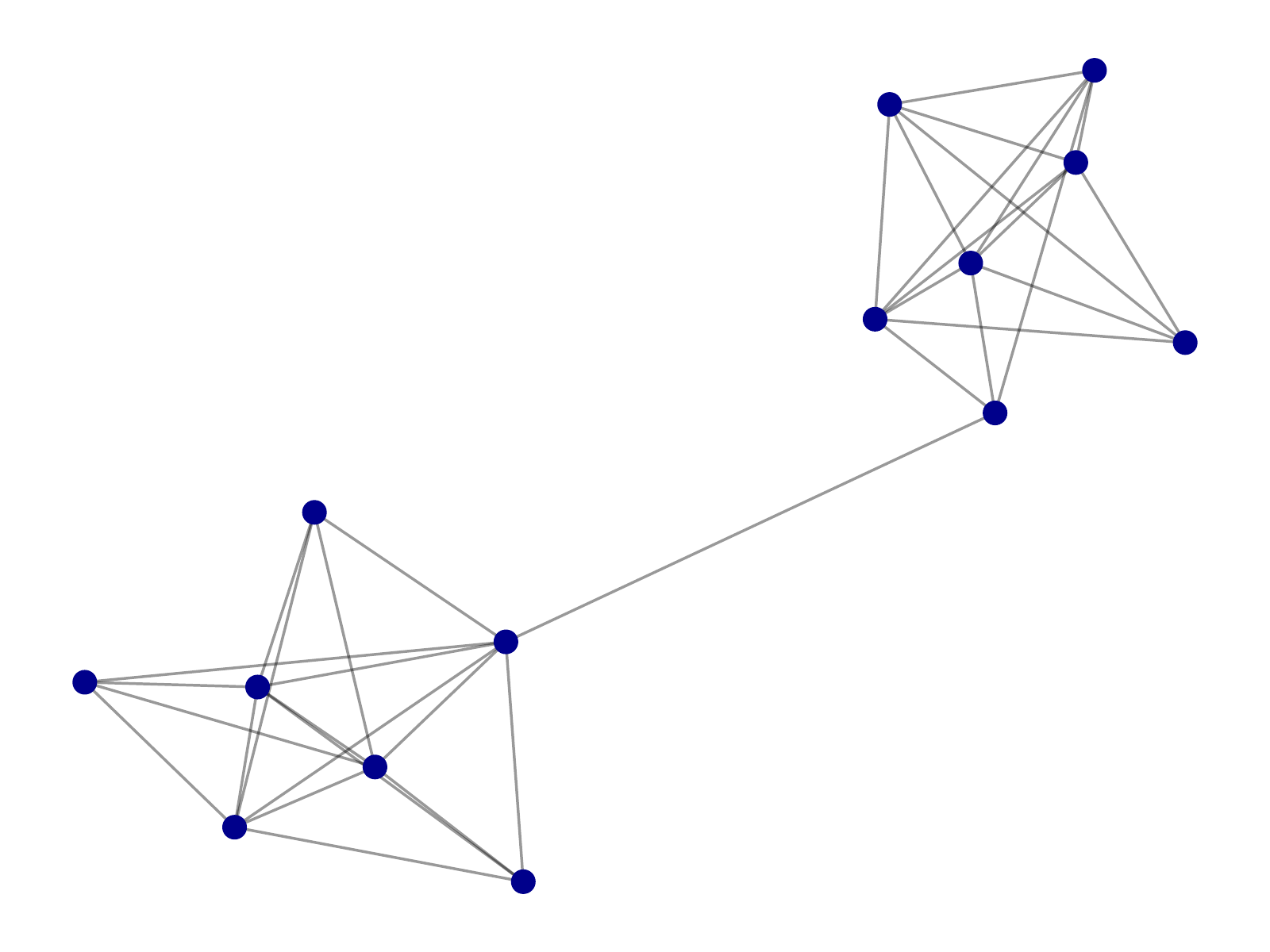}&
    \includegraphics[width=0.09\textwidth]{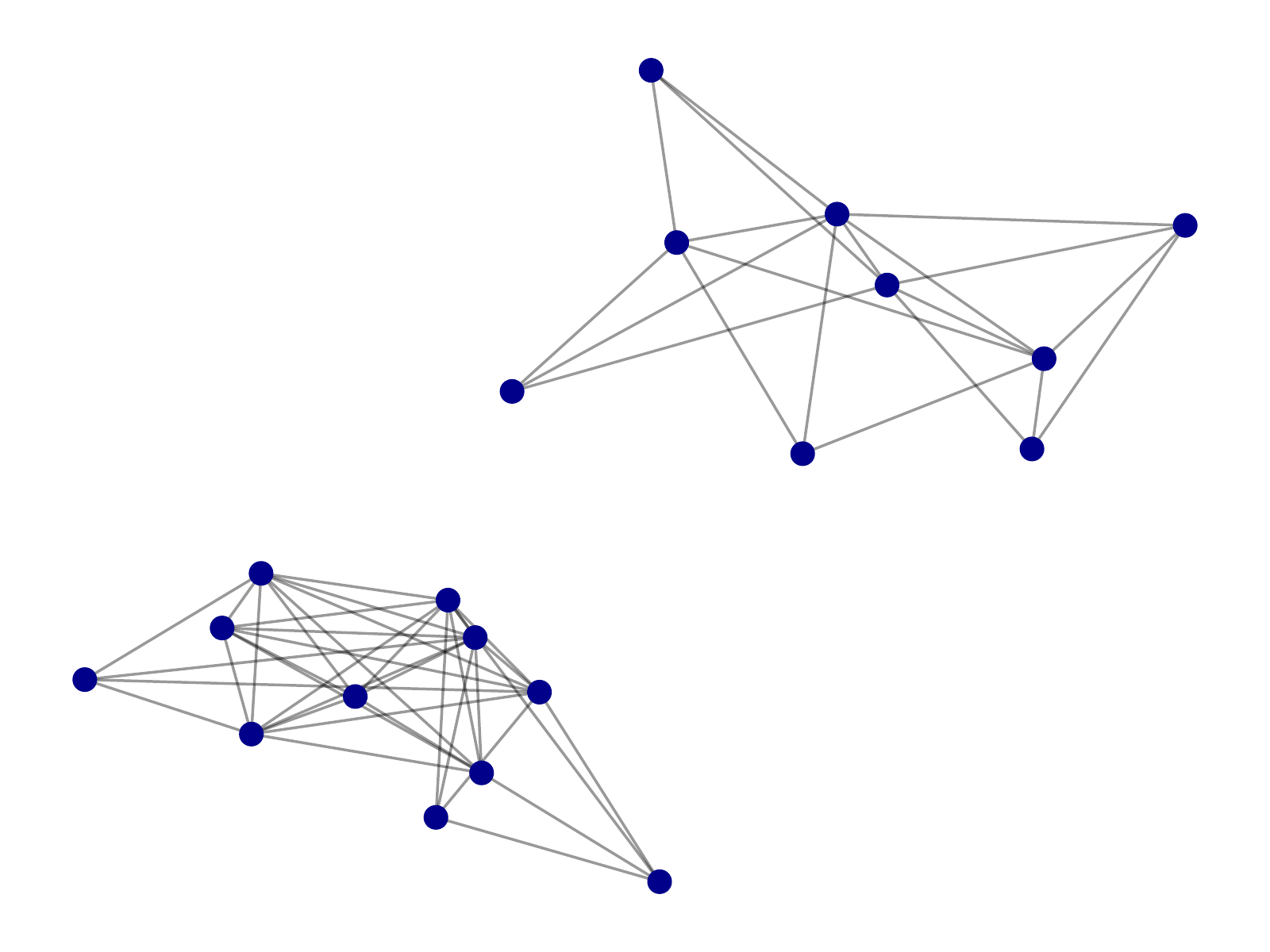}&\includegraphics[width=0.09\textwidth]{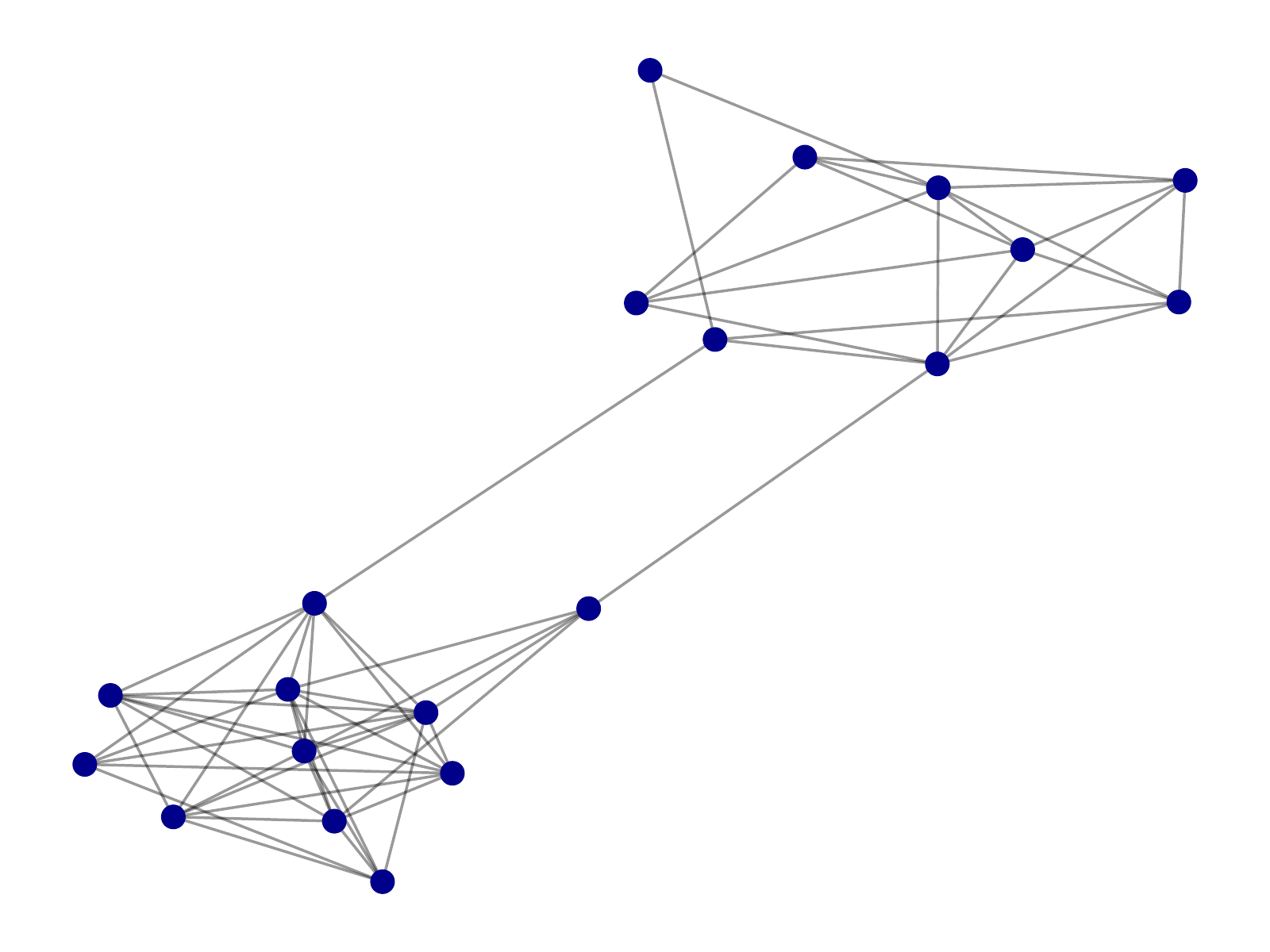}&
    \includegraphics[width=0.09\textwidth]{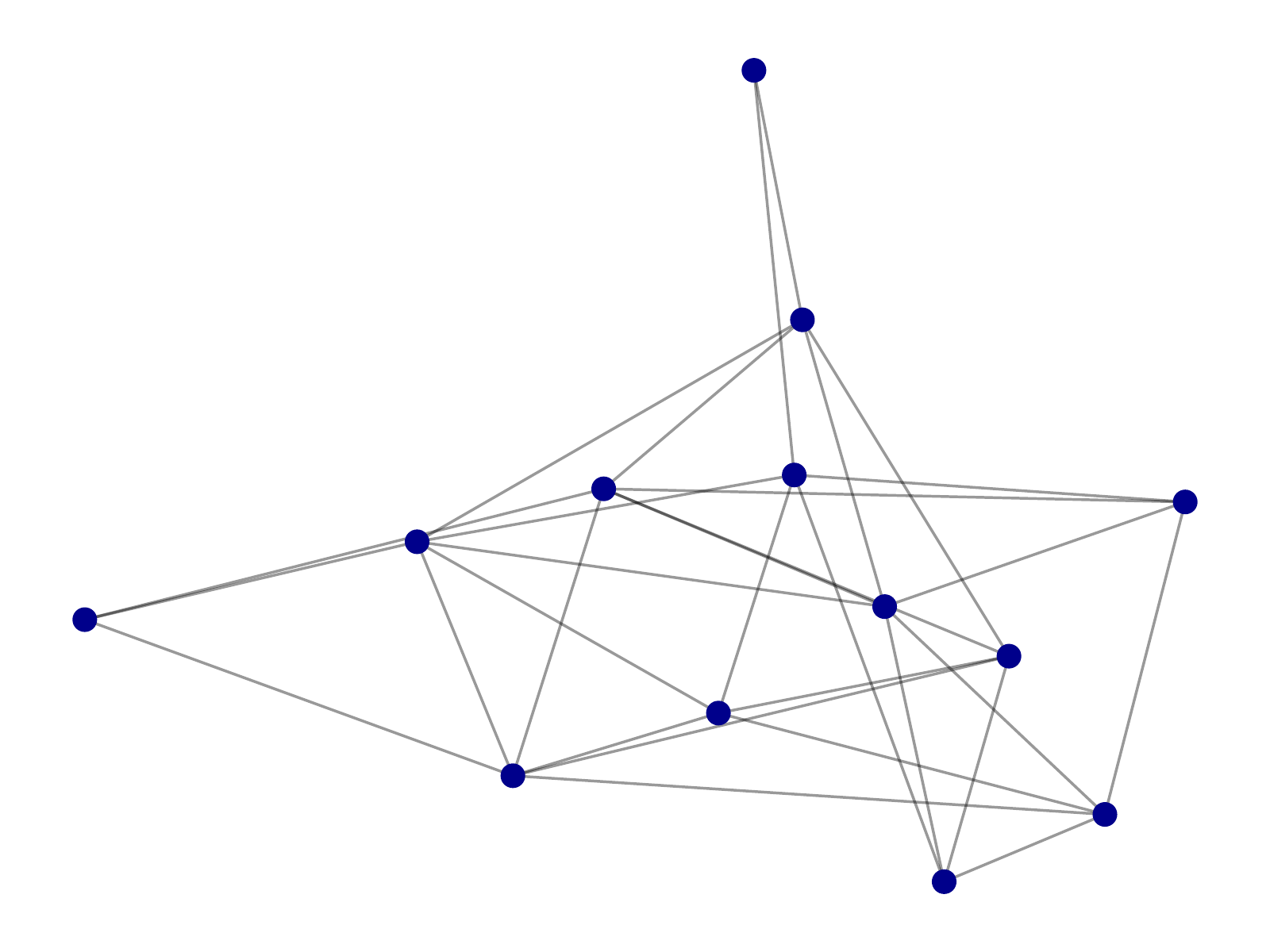}&\includegraphics[width=0.09\textwidth]{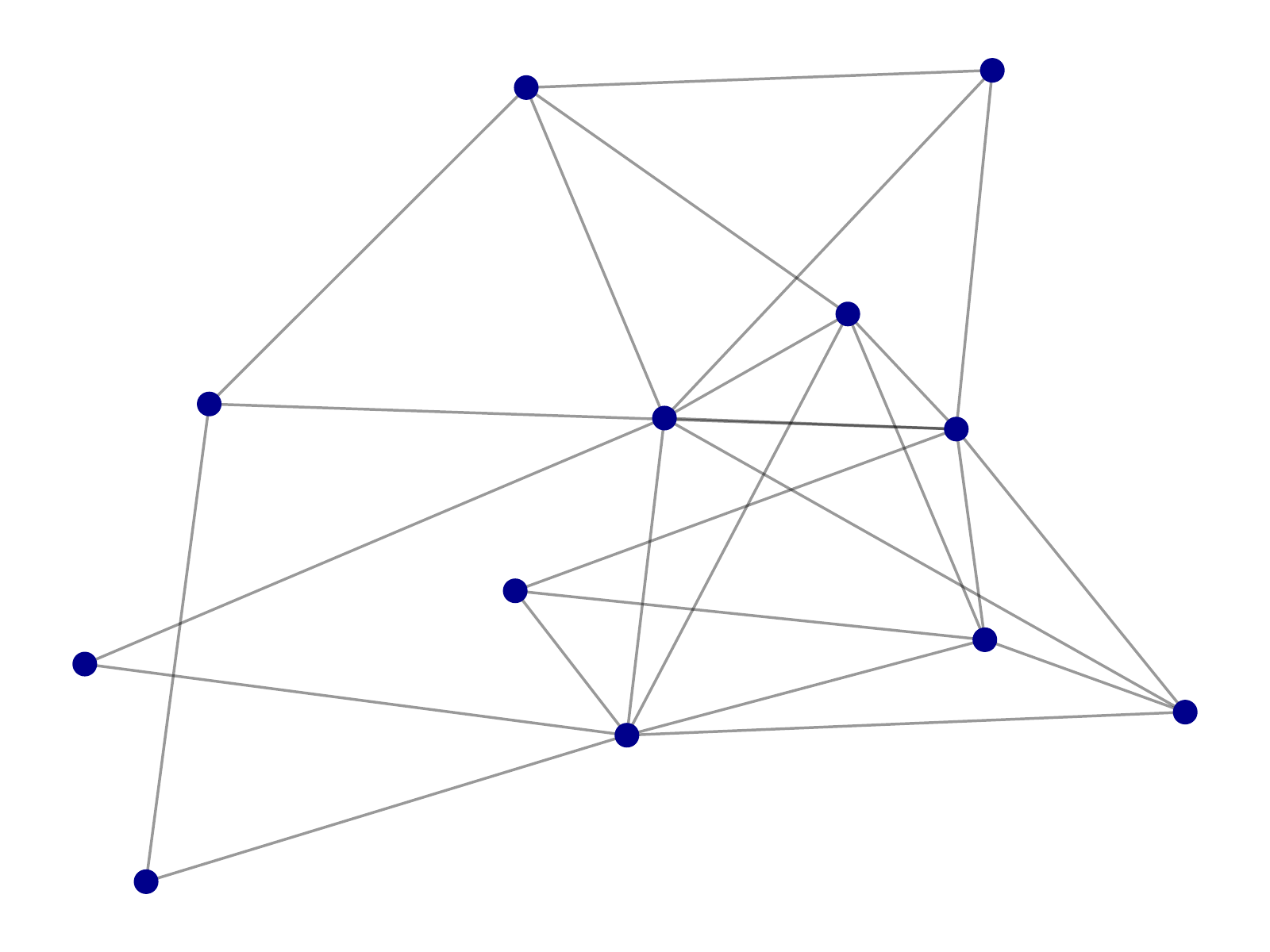}&
    \includegraphics[width=0.09\textwidth]{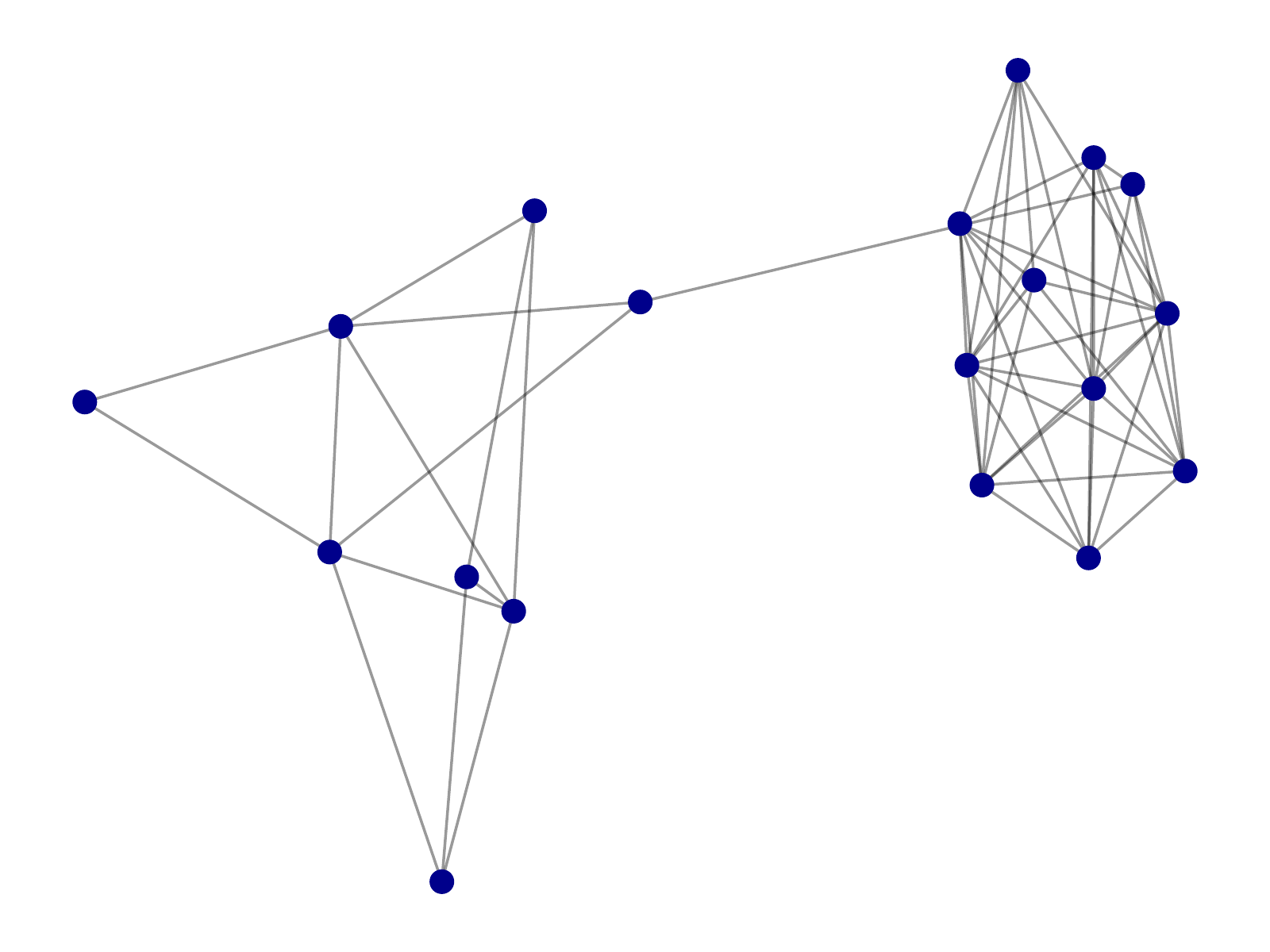}&\includegraphics[width=0.09\textwidth]{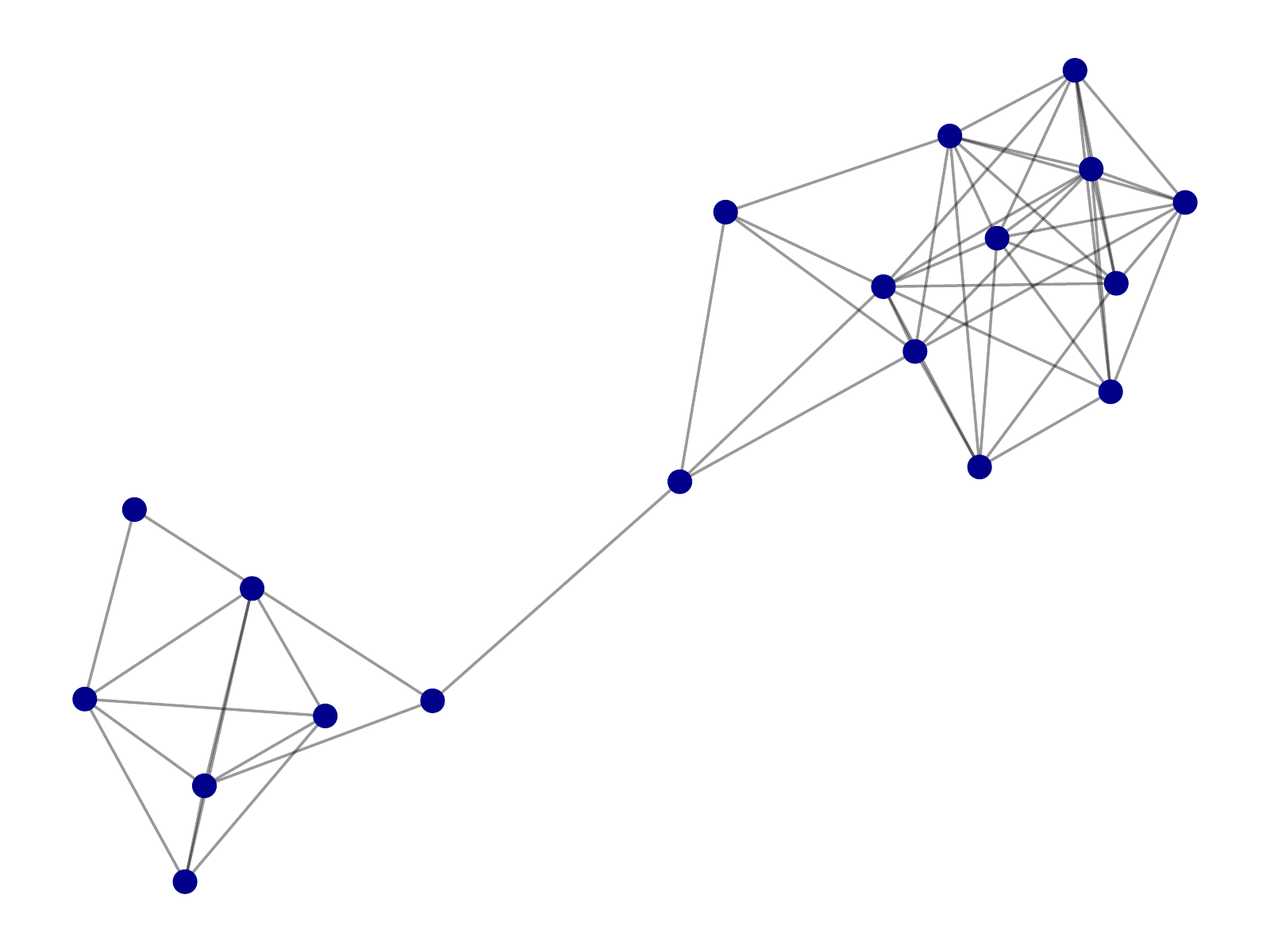}\\
    % \rotatebox[origin=l]{90}{Community}&
    % &
    \includegraphics[width=0.09\textwidth]{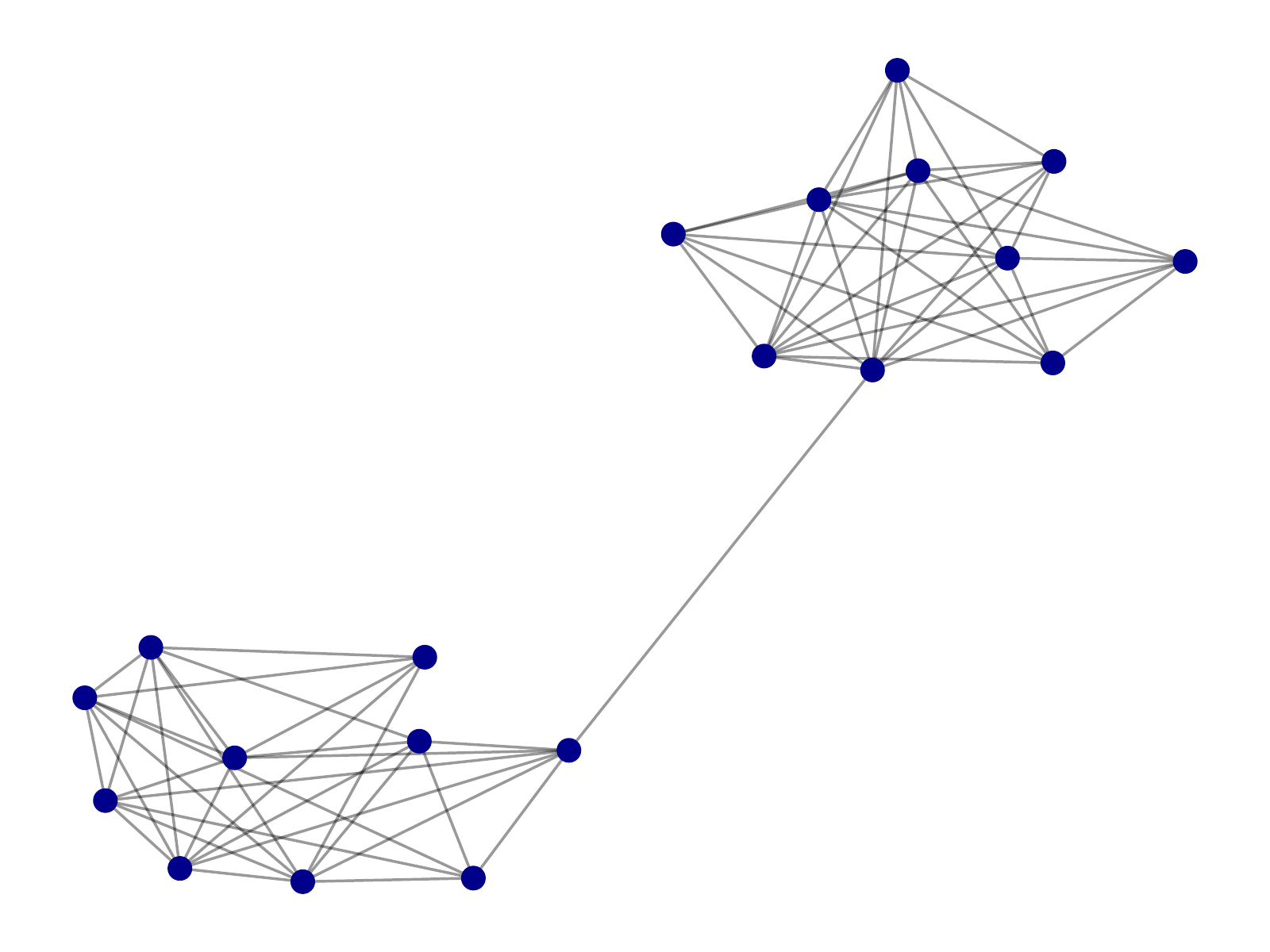}&\includegraphics[width=0.09\textwidth]{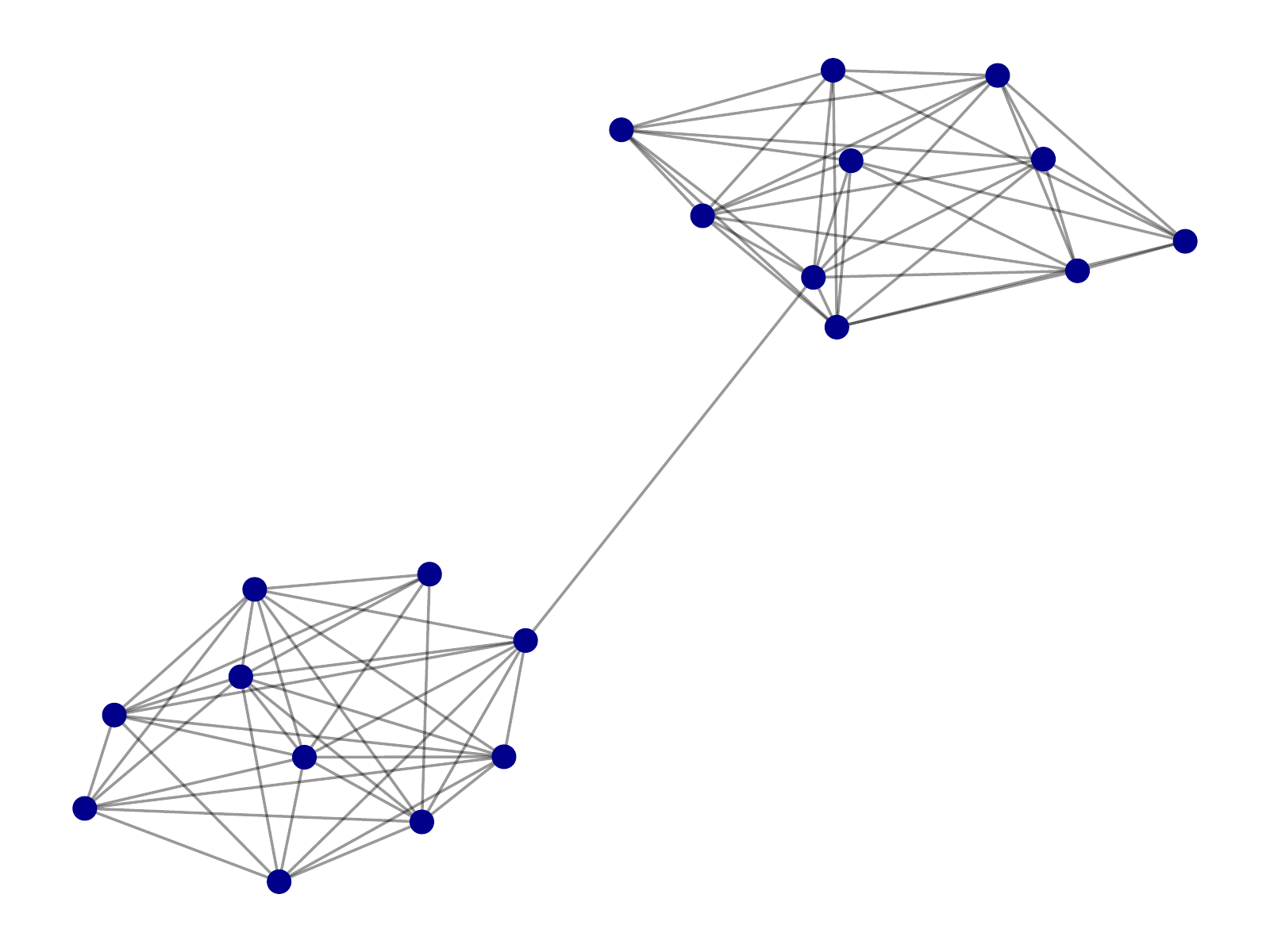}&
    \includegraphics[width=0.09\textwidth]{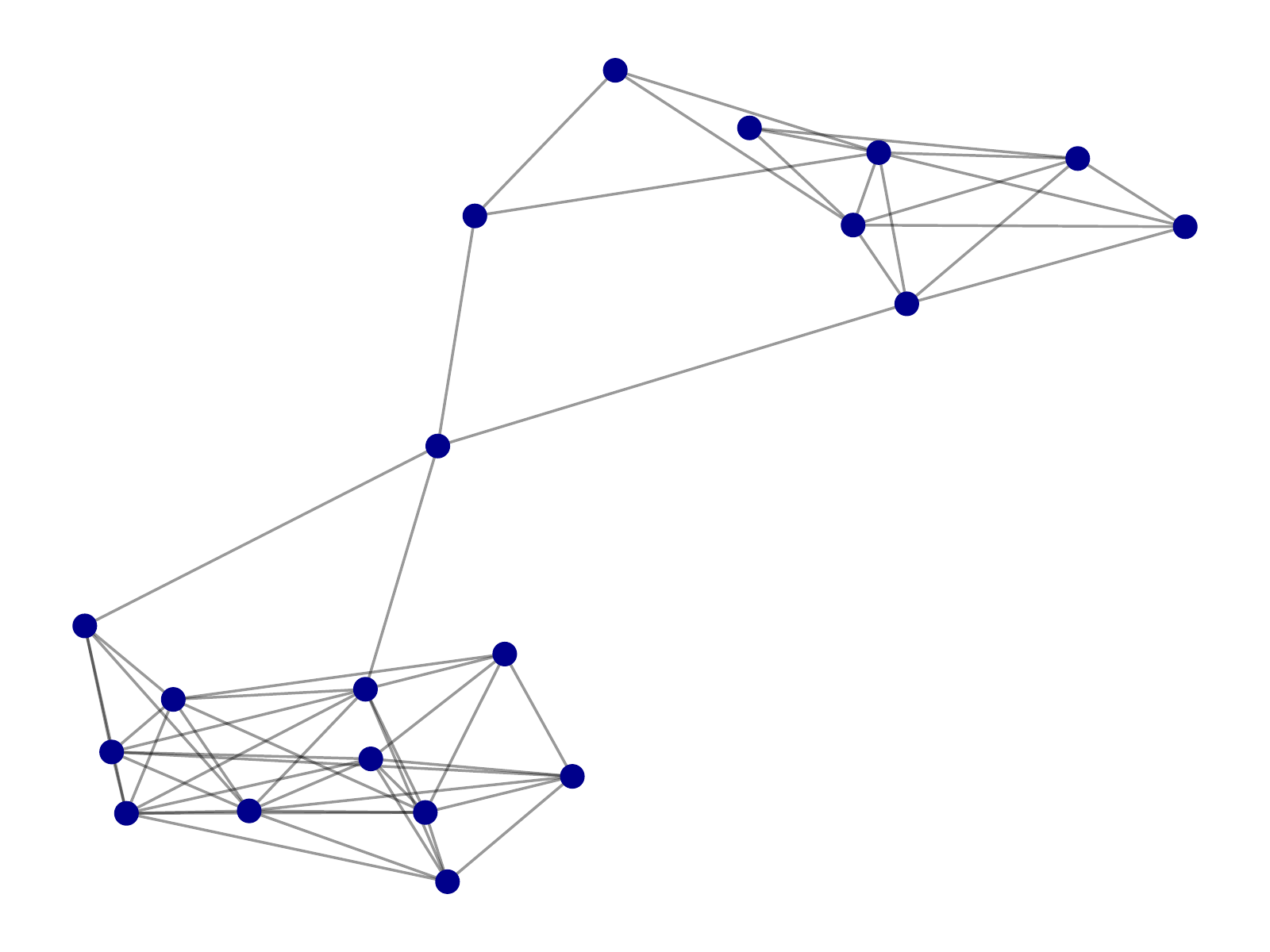}&\includegraphics[width=0.09\textwidth]{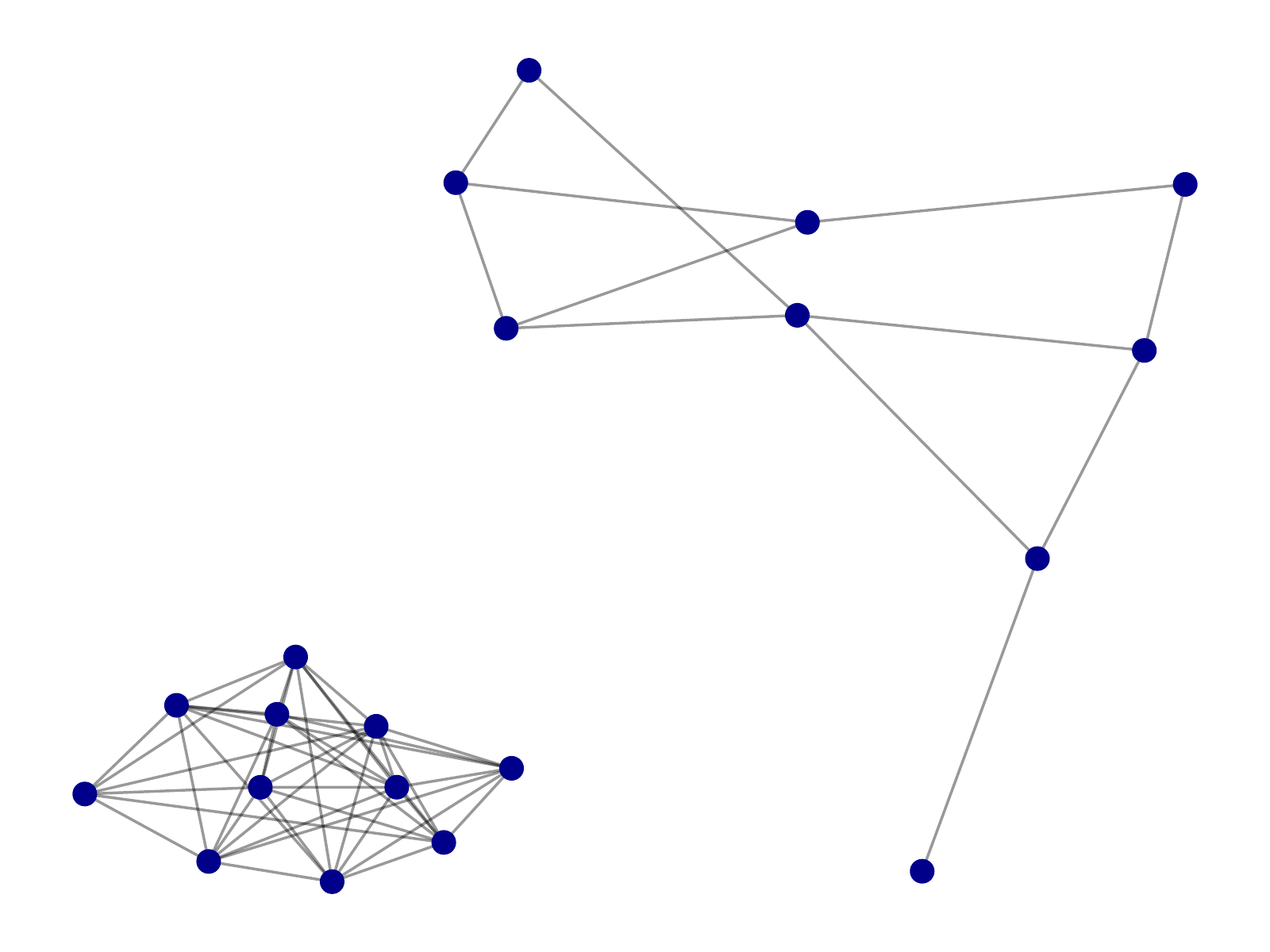}&
    \includegraphics[width=0.09\textwidth]{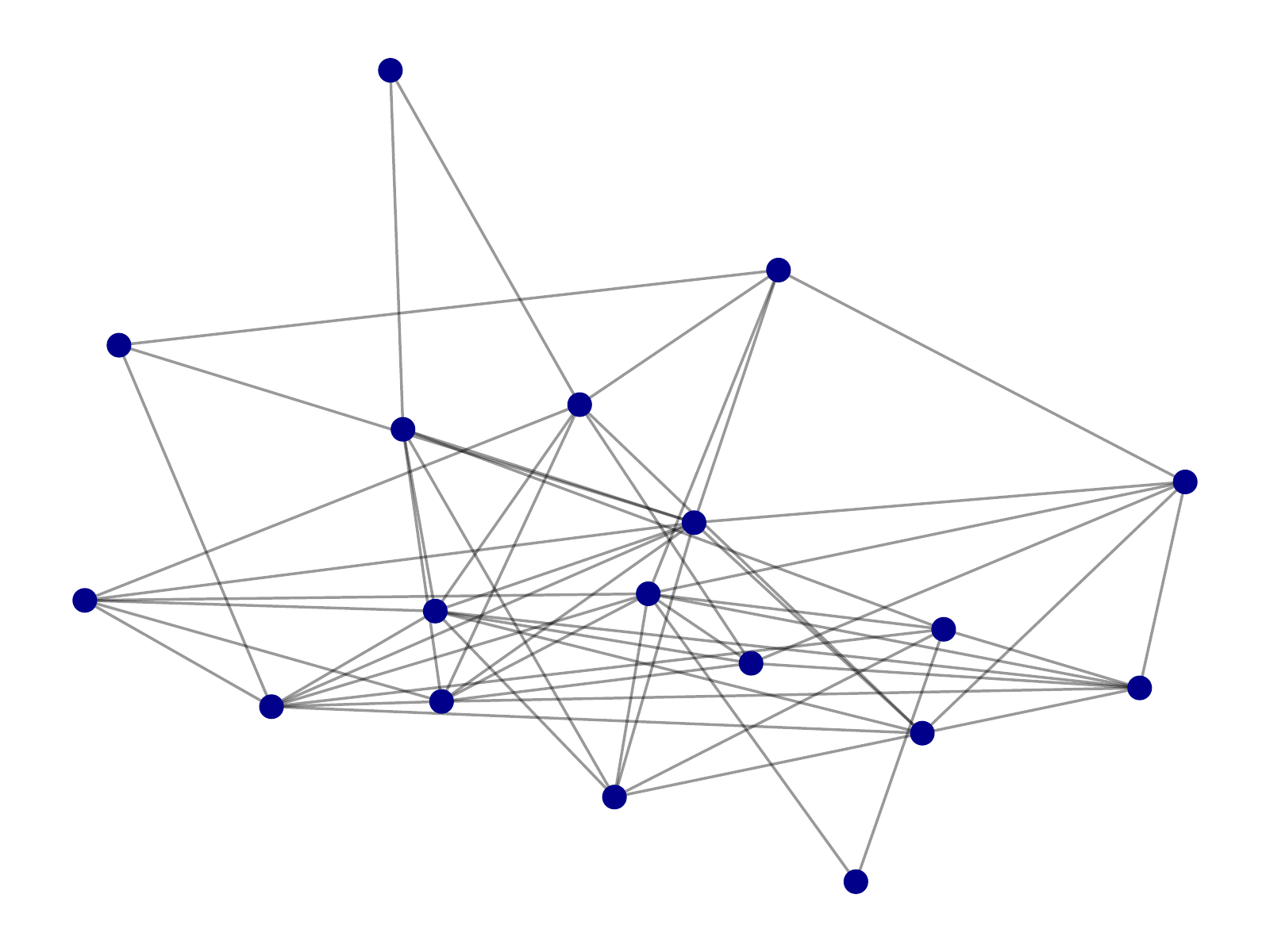}&\includegraphics[width=0.09\textwidth]{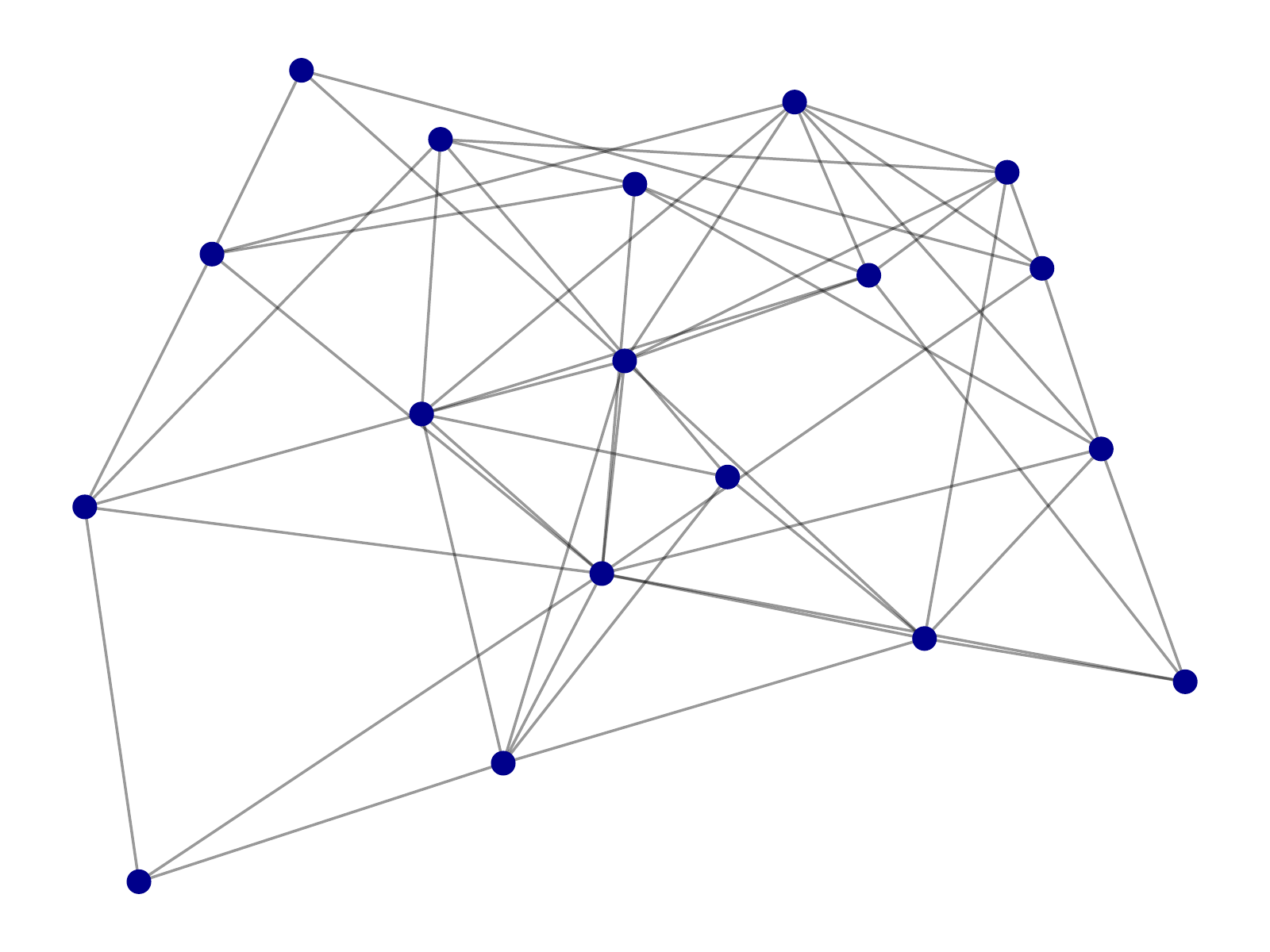}&
    \includegraphics[width=0.09\textwidth]{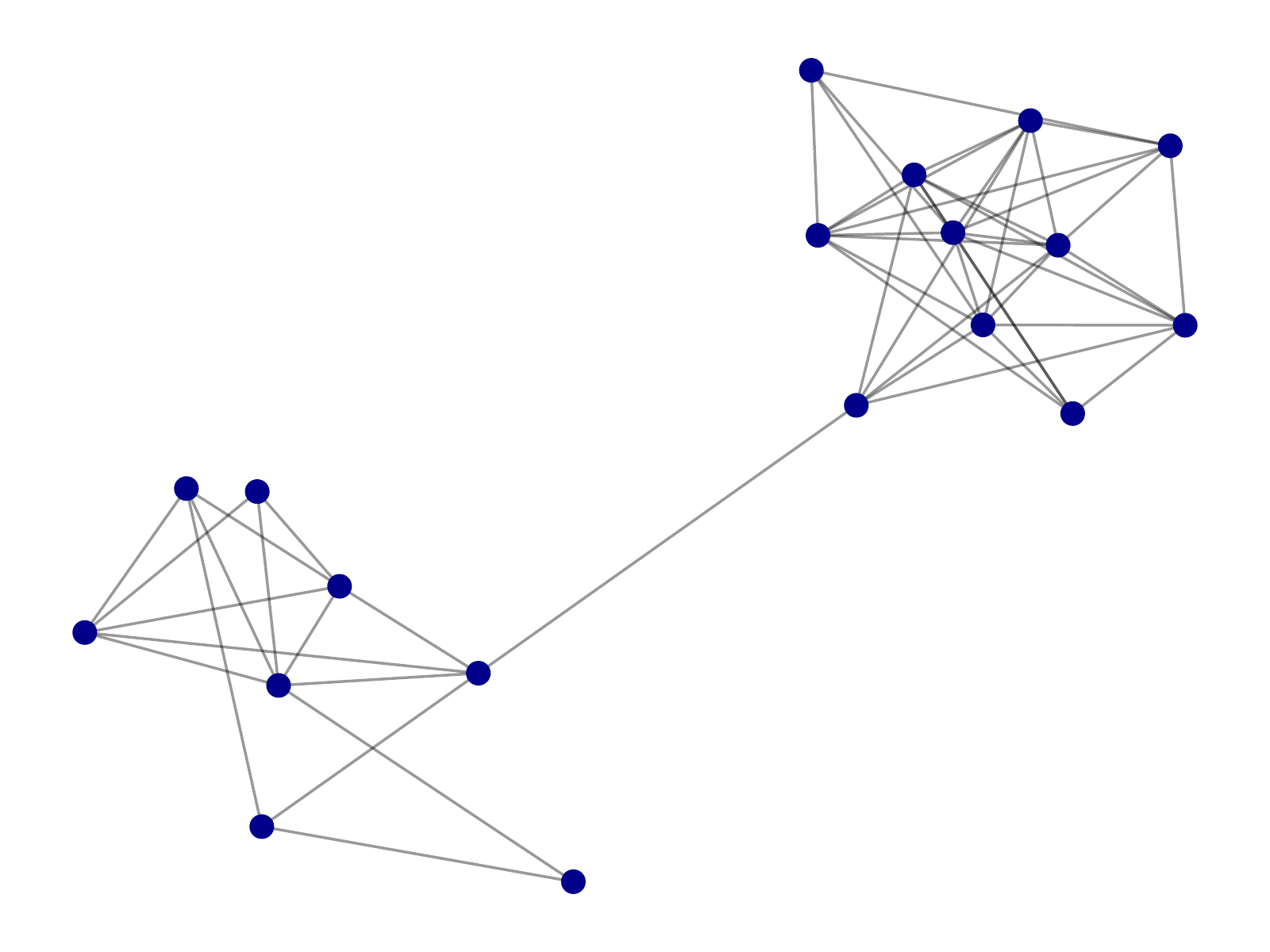}&\includegraphics[width=0.09\textwidth]{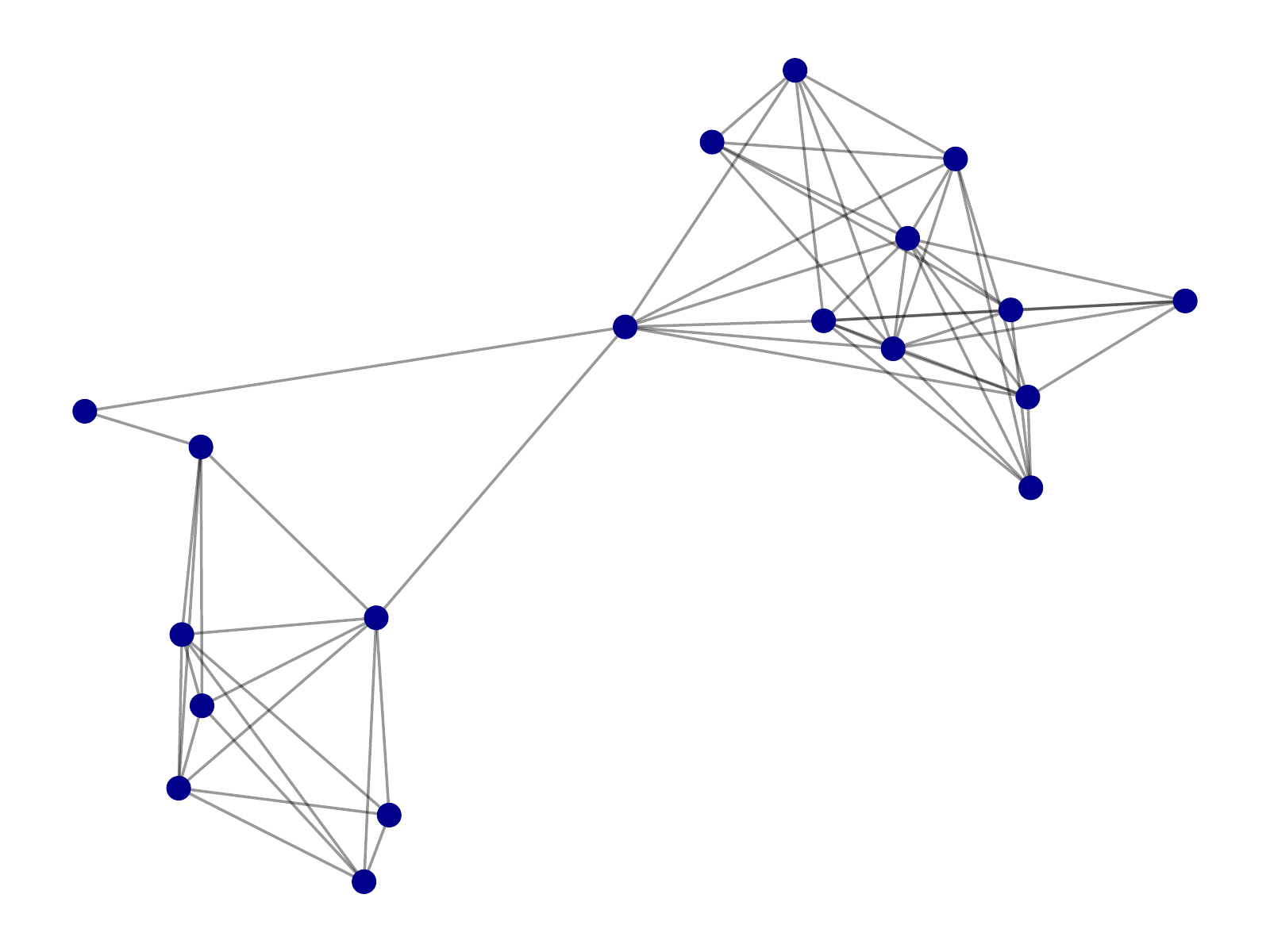}\\
    \multicolumn{8}{c}{Community-small}\\[10pt]
    \includegraphics[width=0.09\textwidth]{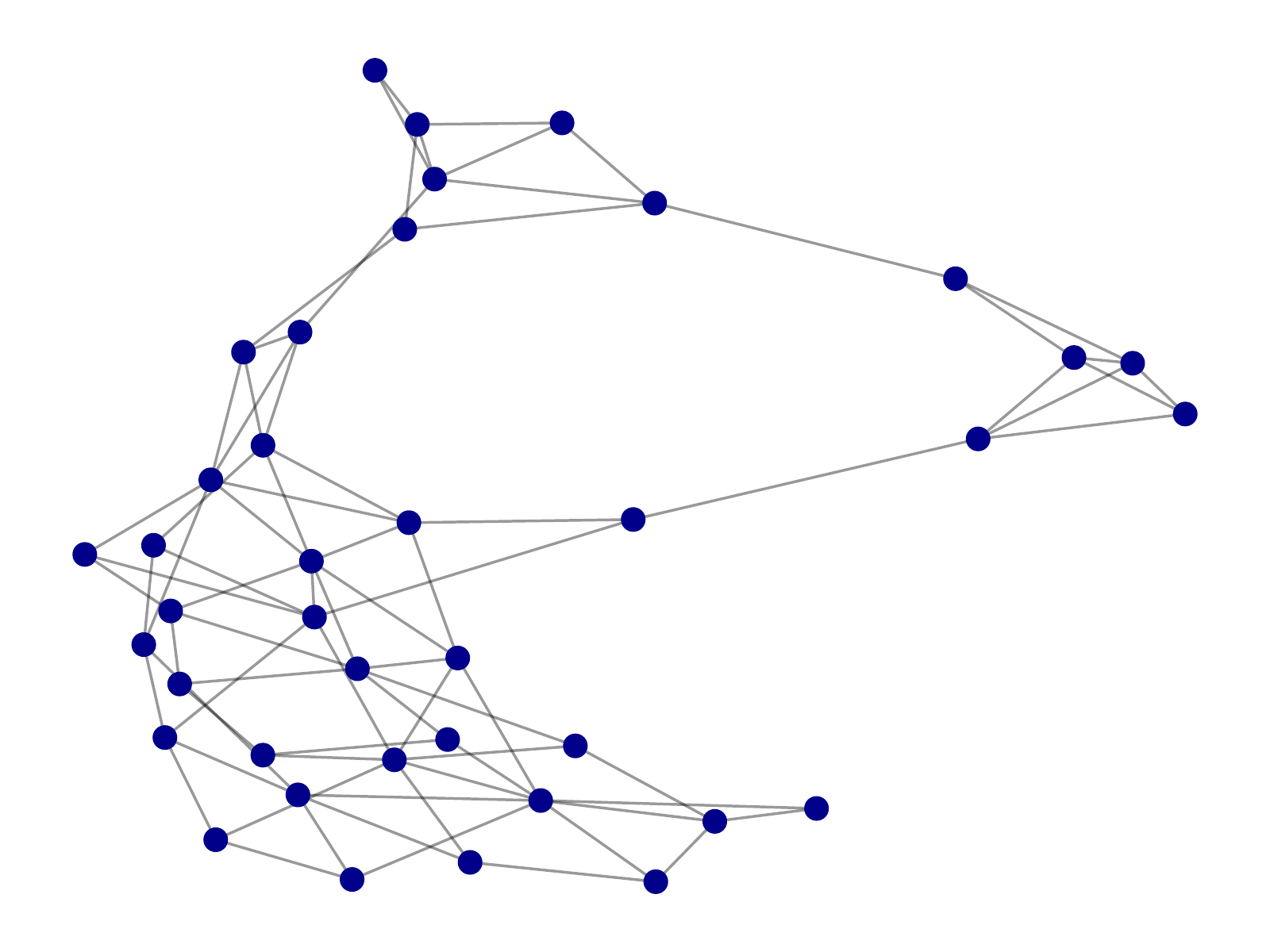}&\includegraphics[width=0.09\textwidth]{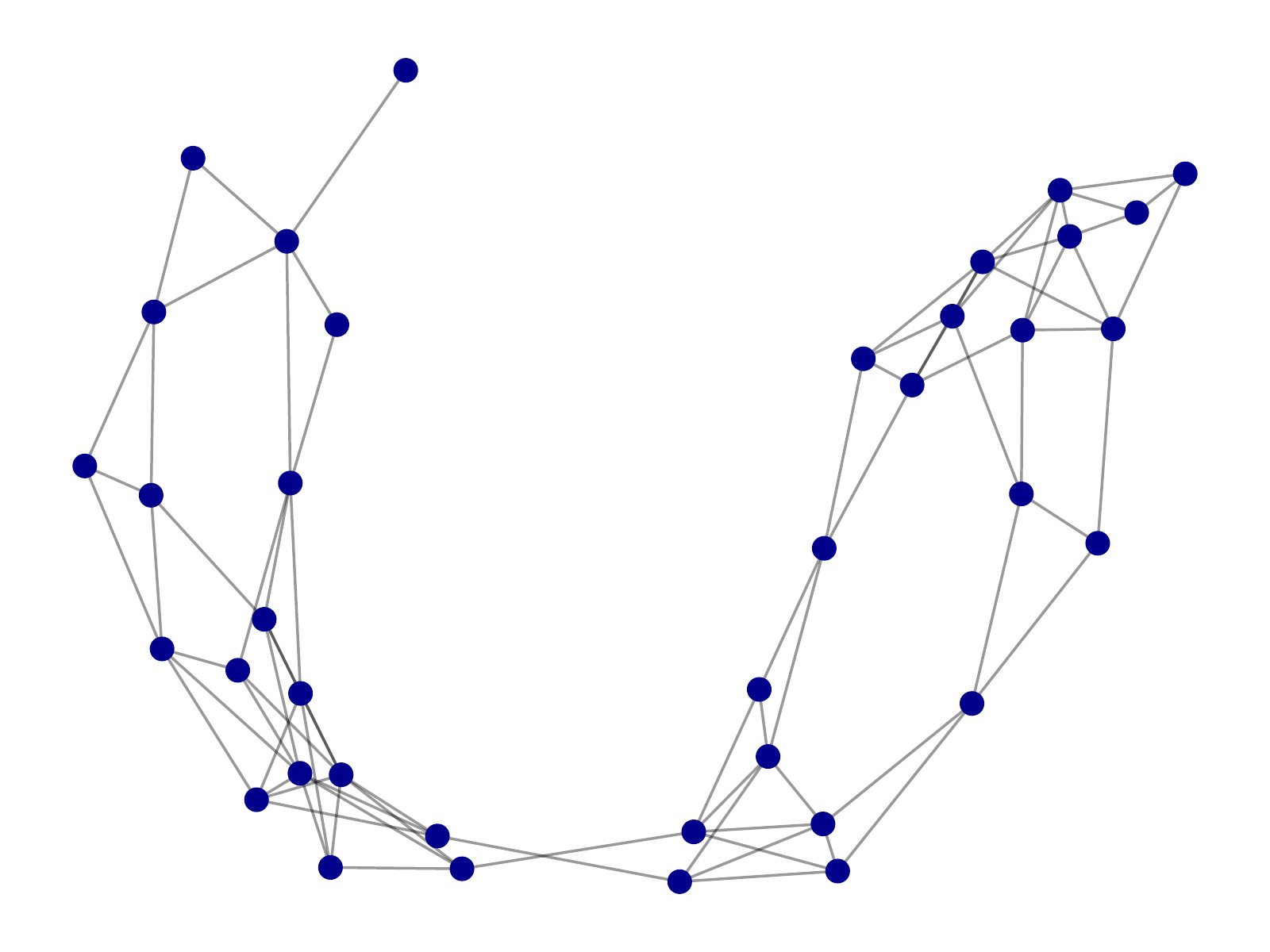}&
    \includegraphics[width=0.09\textwidth]{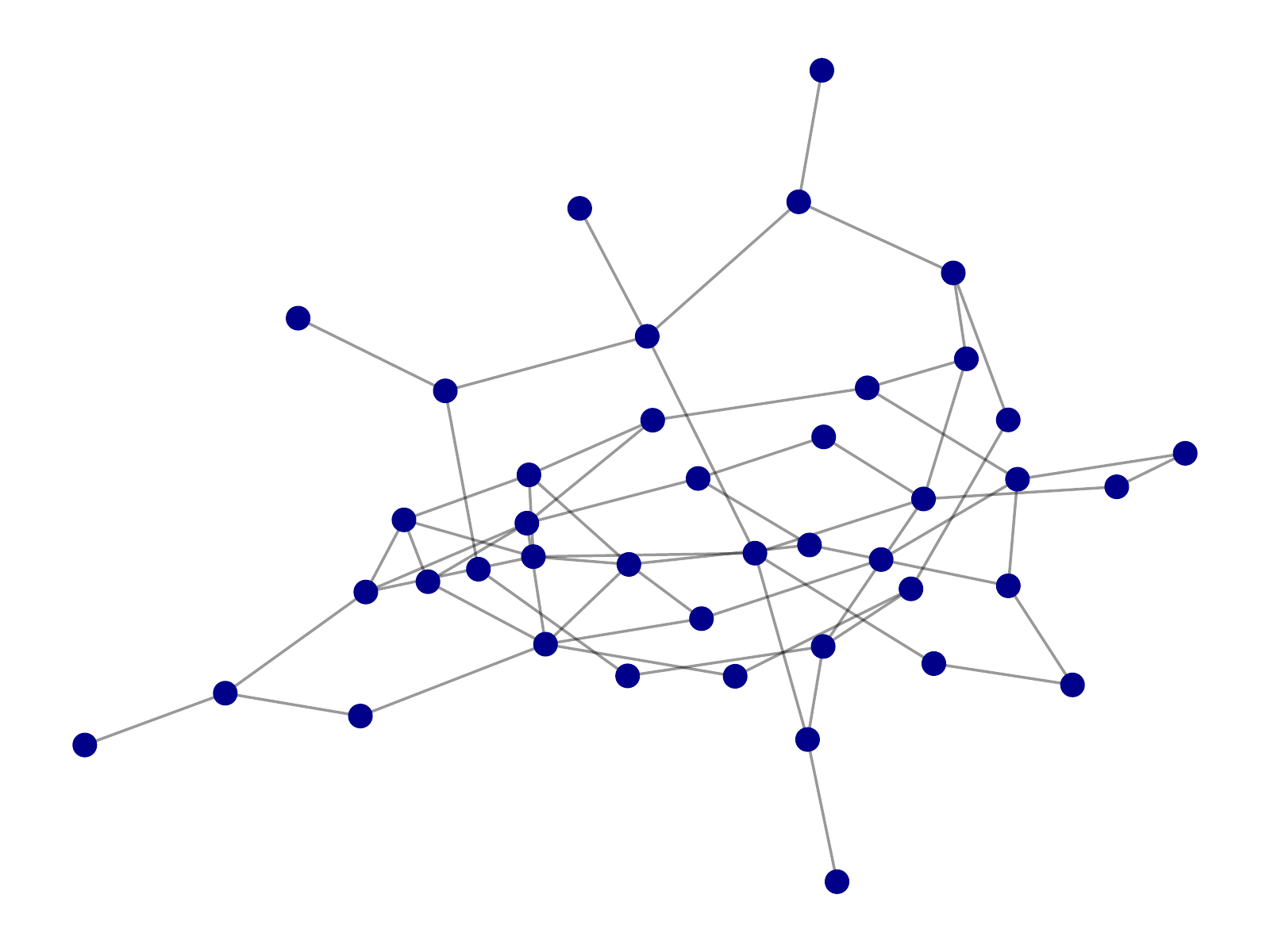}&\includegraphics[width=0.09\textwidth]{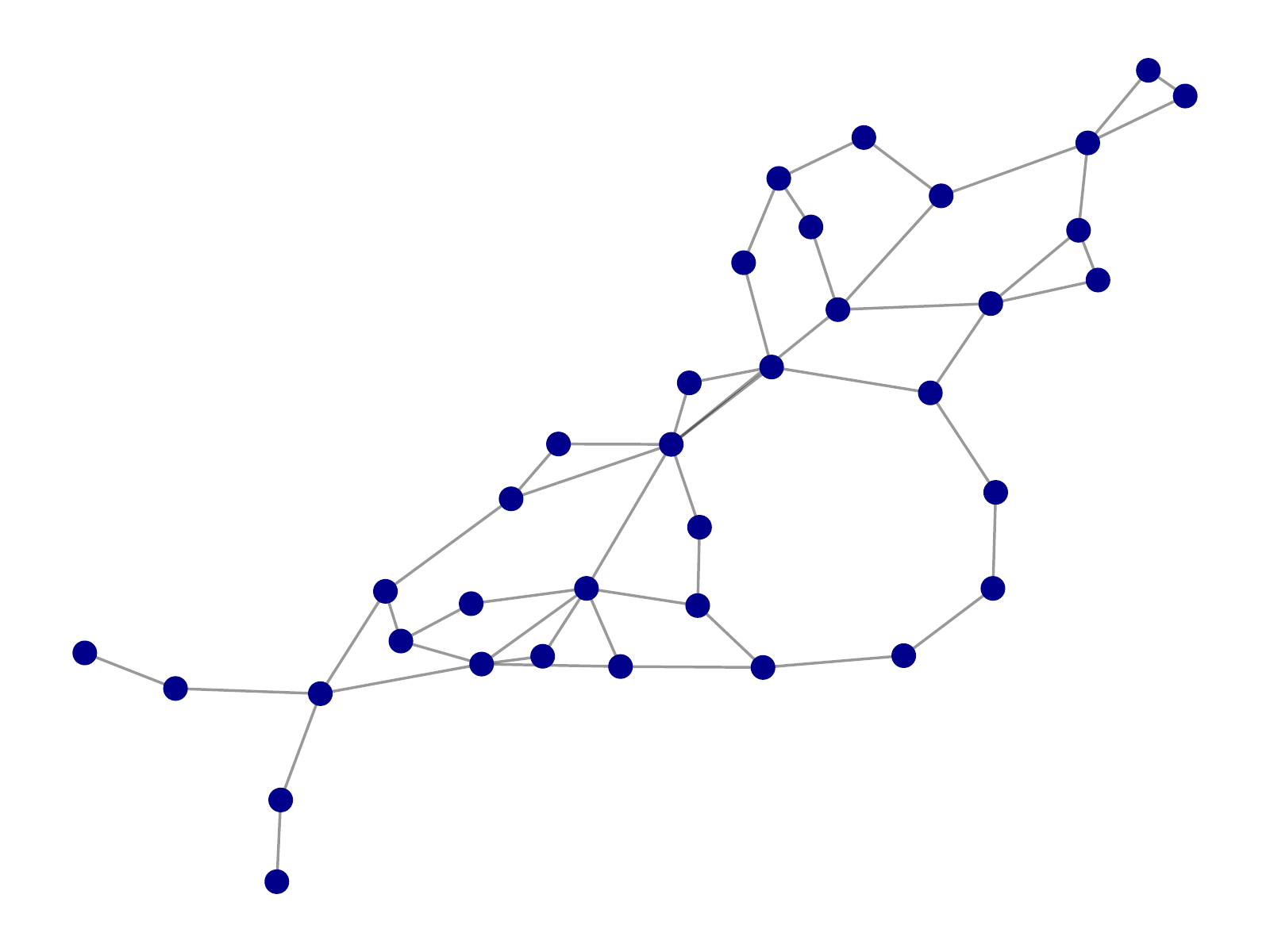}&
    \includegraphics[width=0.09\textwidth]{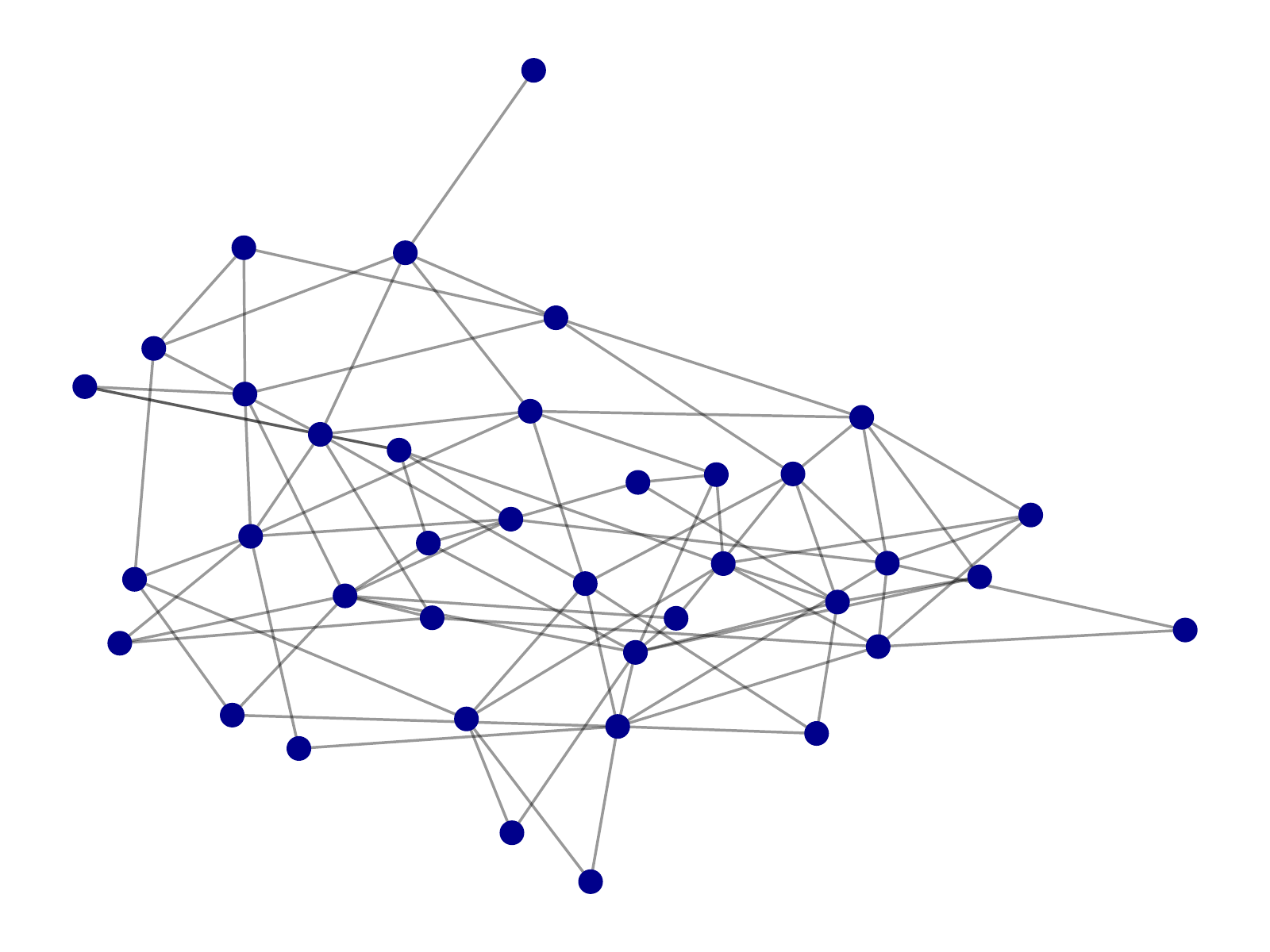}&\includegraphics[width=0.09\textwidth]{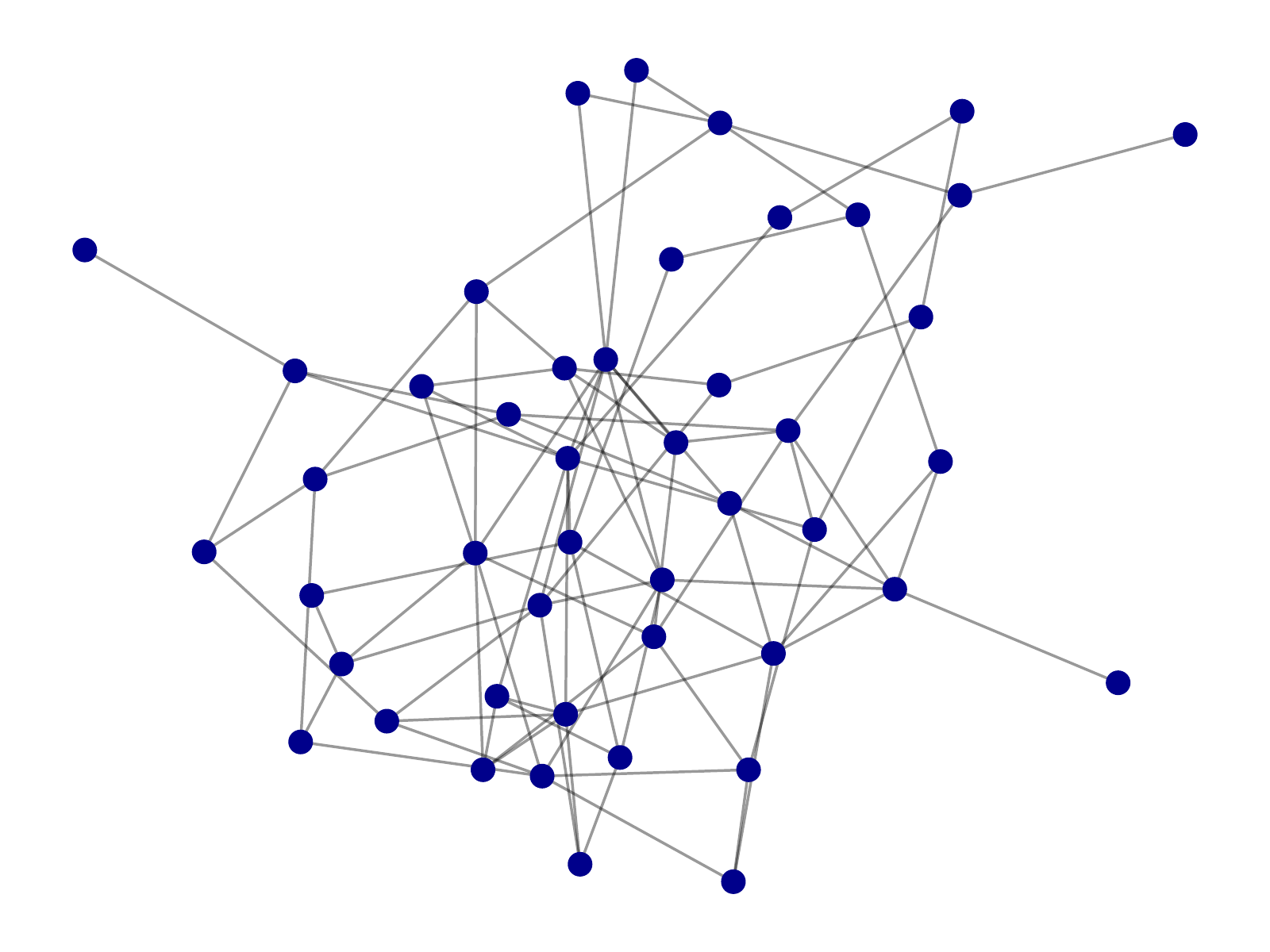}&
    \includegraphics[width=0.09\textwidth]{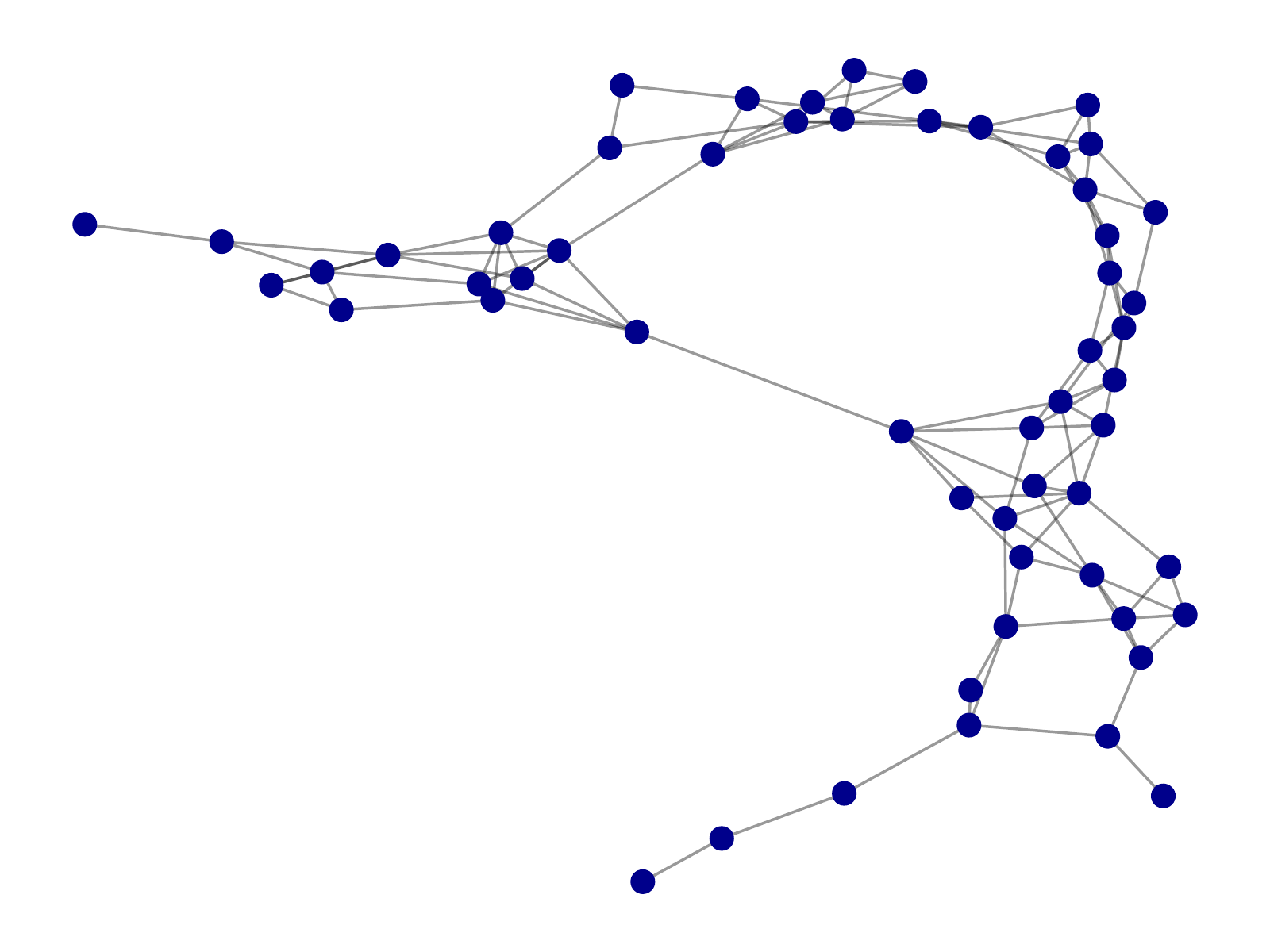}&\includegraphics[width=0.09\textwidth]{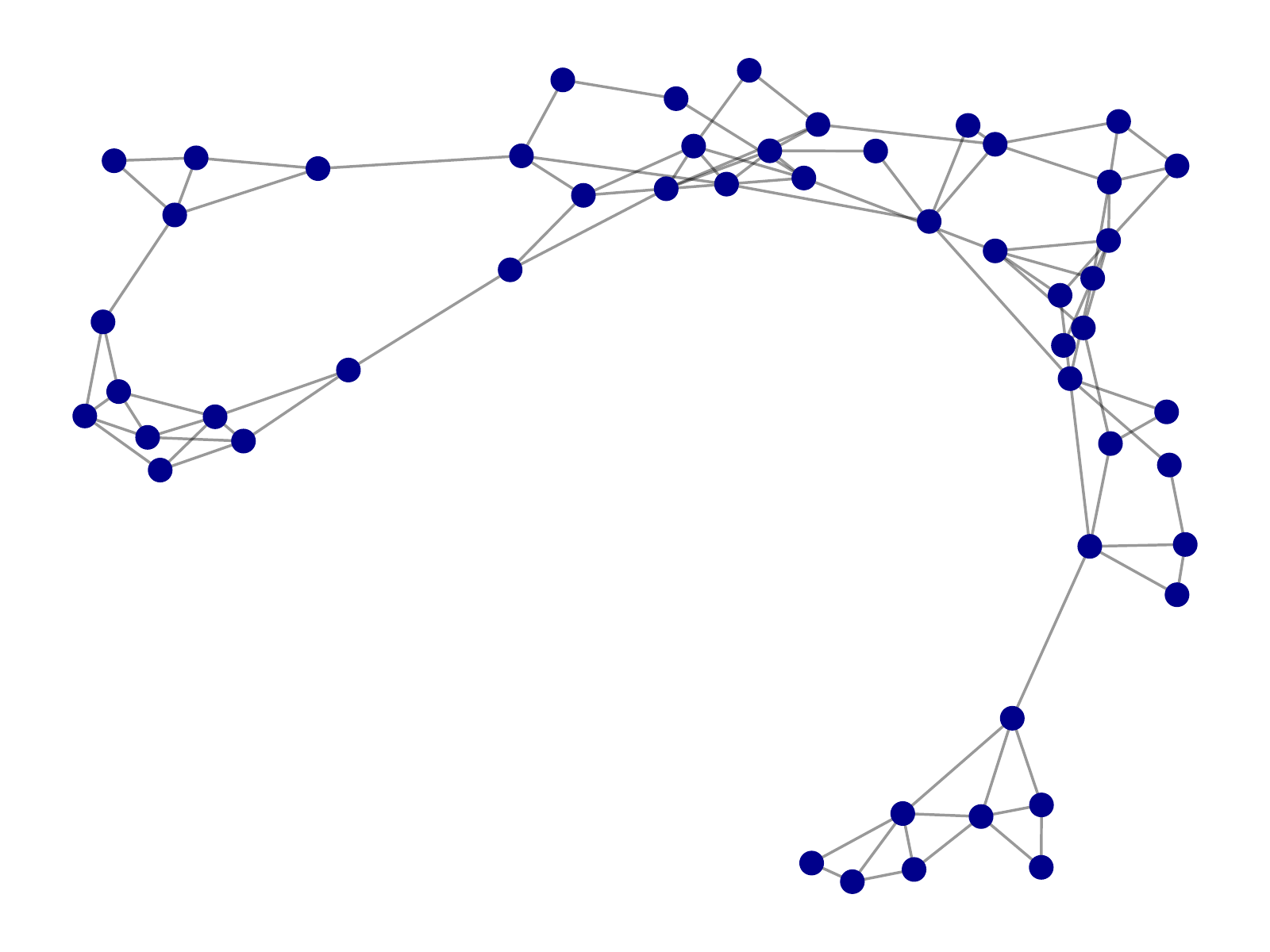}\\
    % &
    \includegraphics[width=0.09\textwidth]{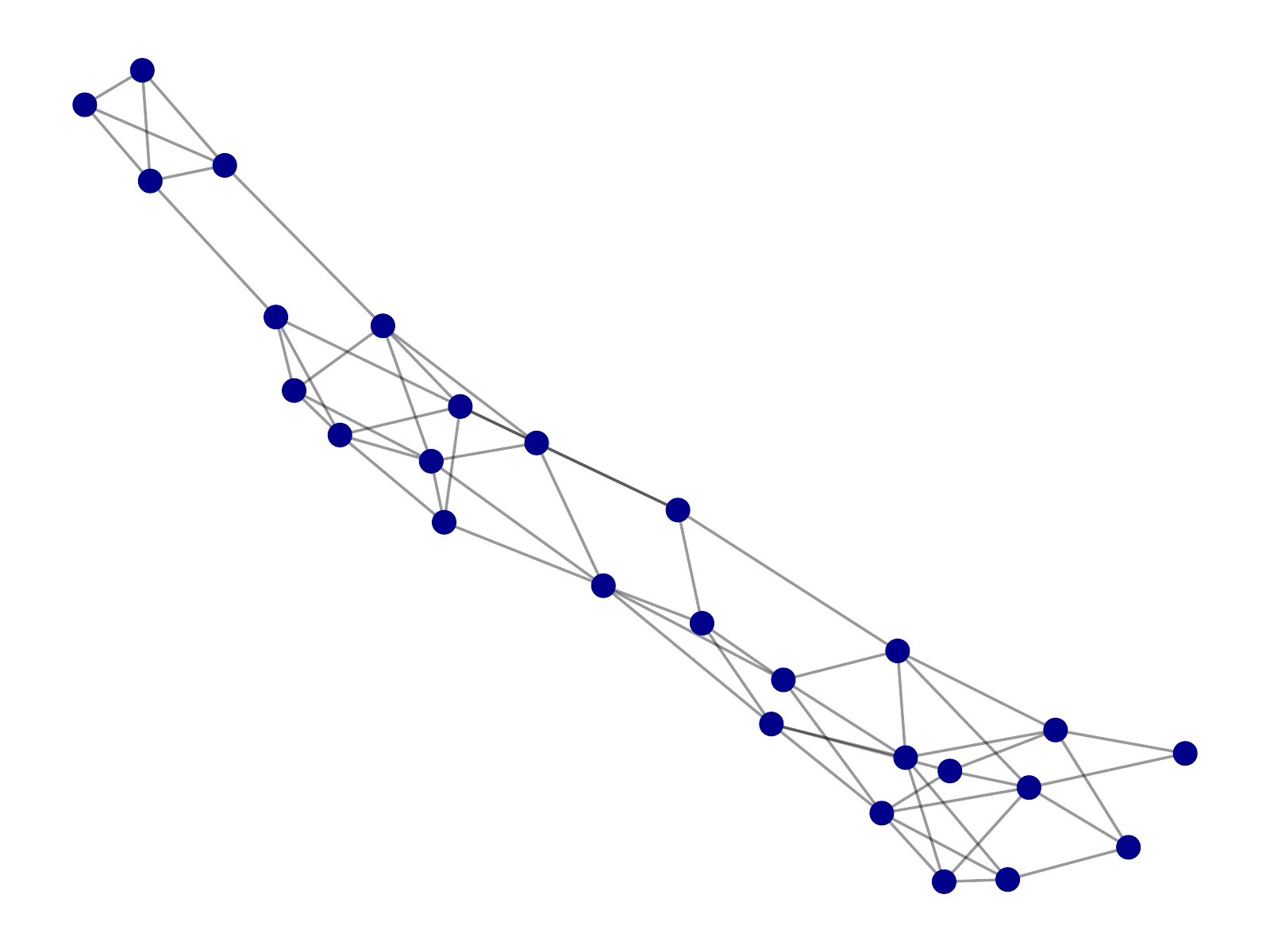}&\includegraphics[width=0.09\textwidth]{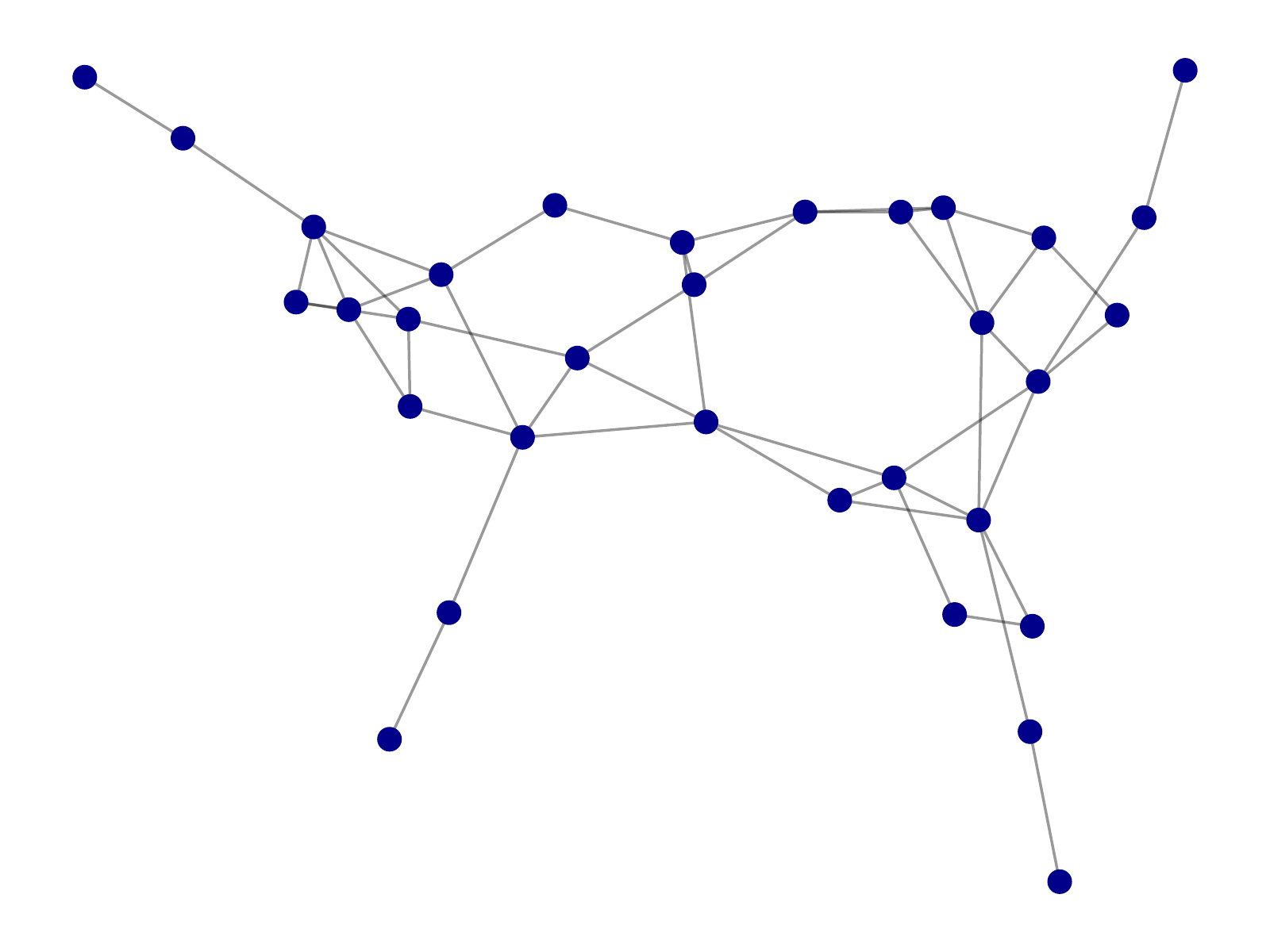}&
    \includegraphics[width=0.09\textwidth]{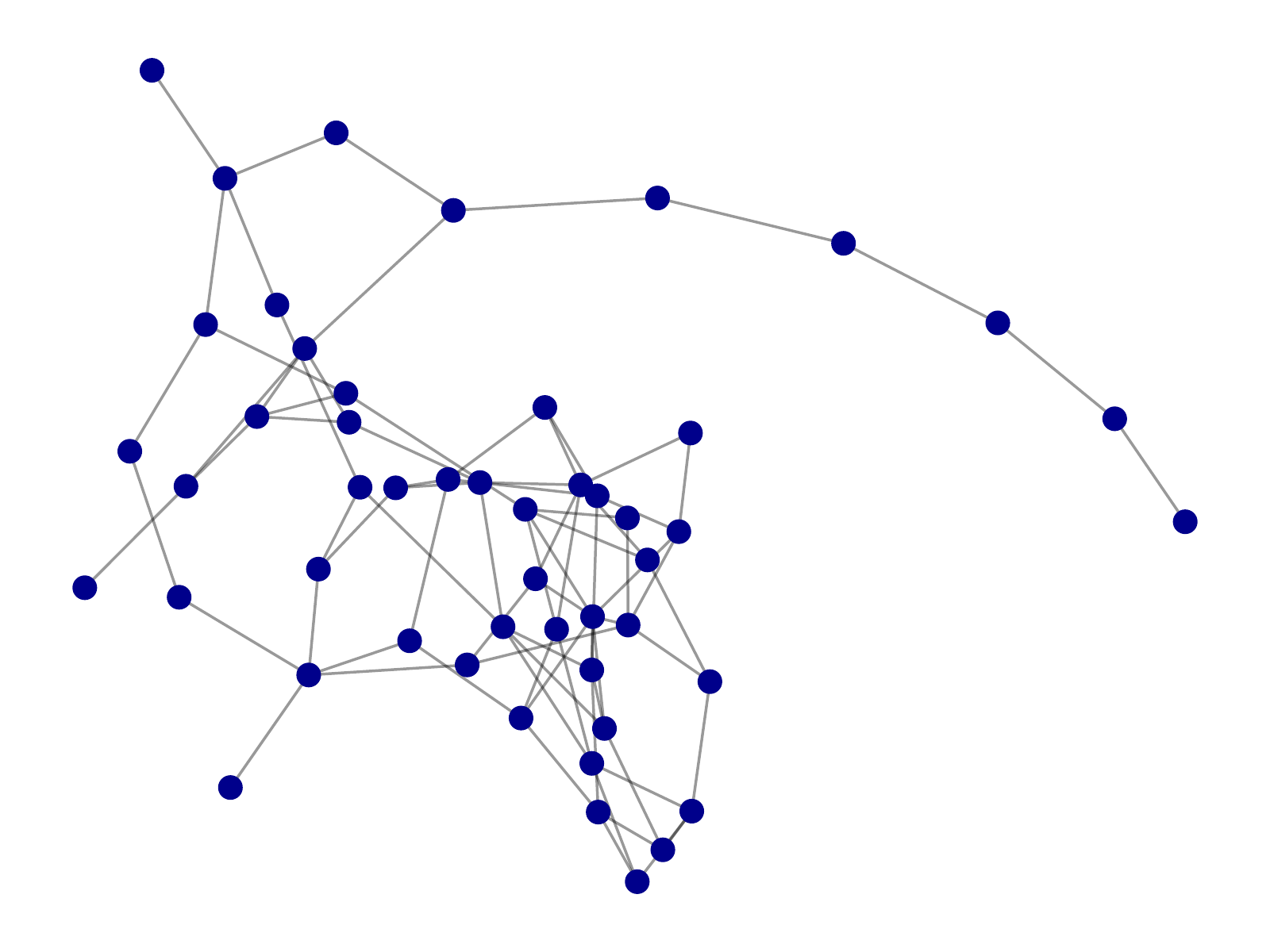}&\includegraphics[width=0.09\textwidth]{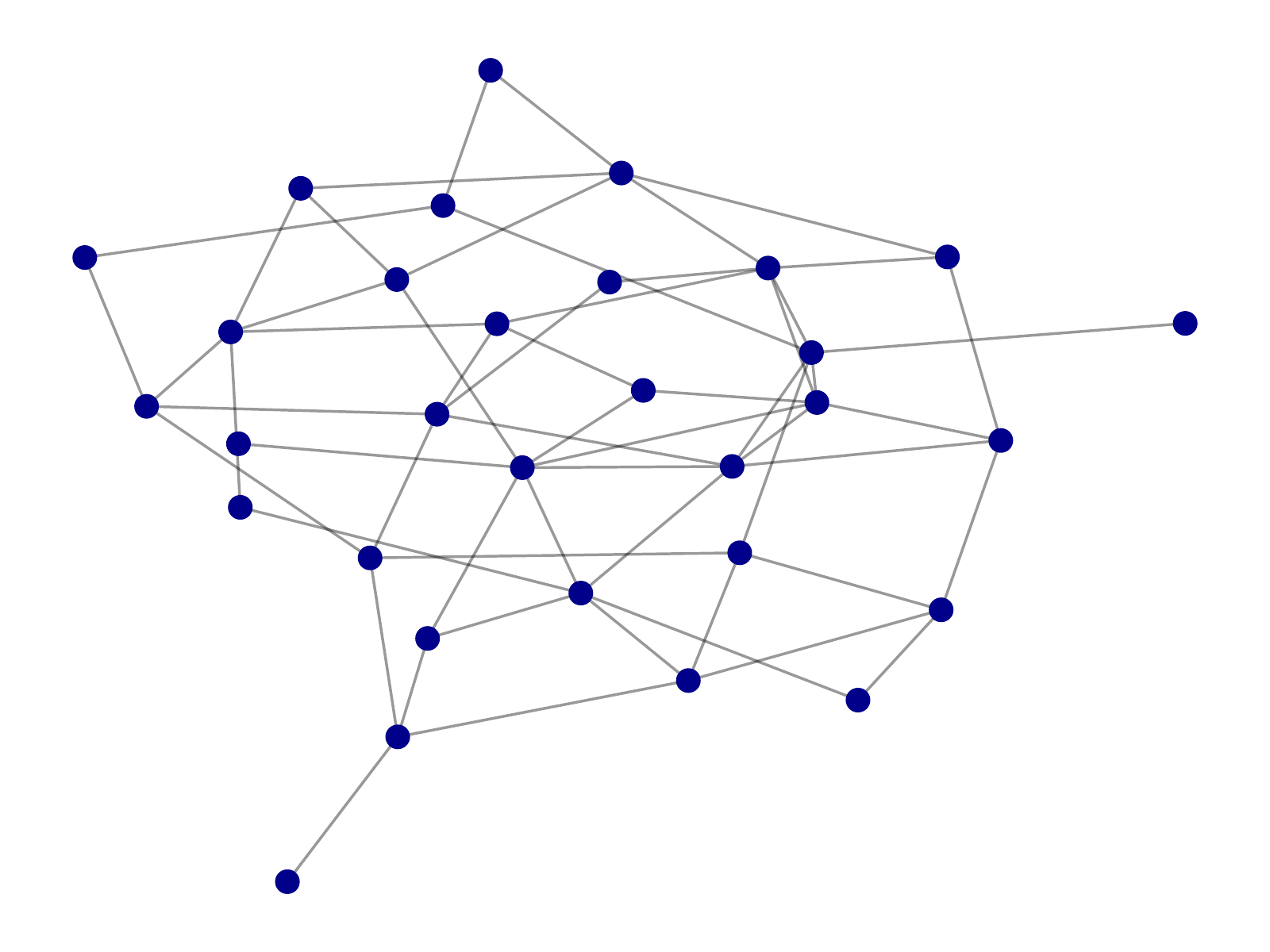}&
    \includegraphics[width=0.09\textwidth]{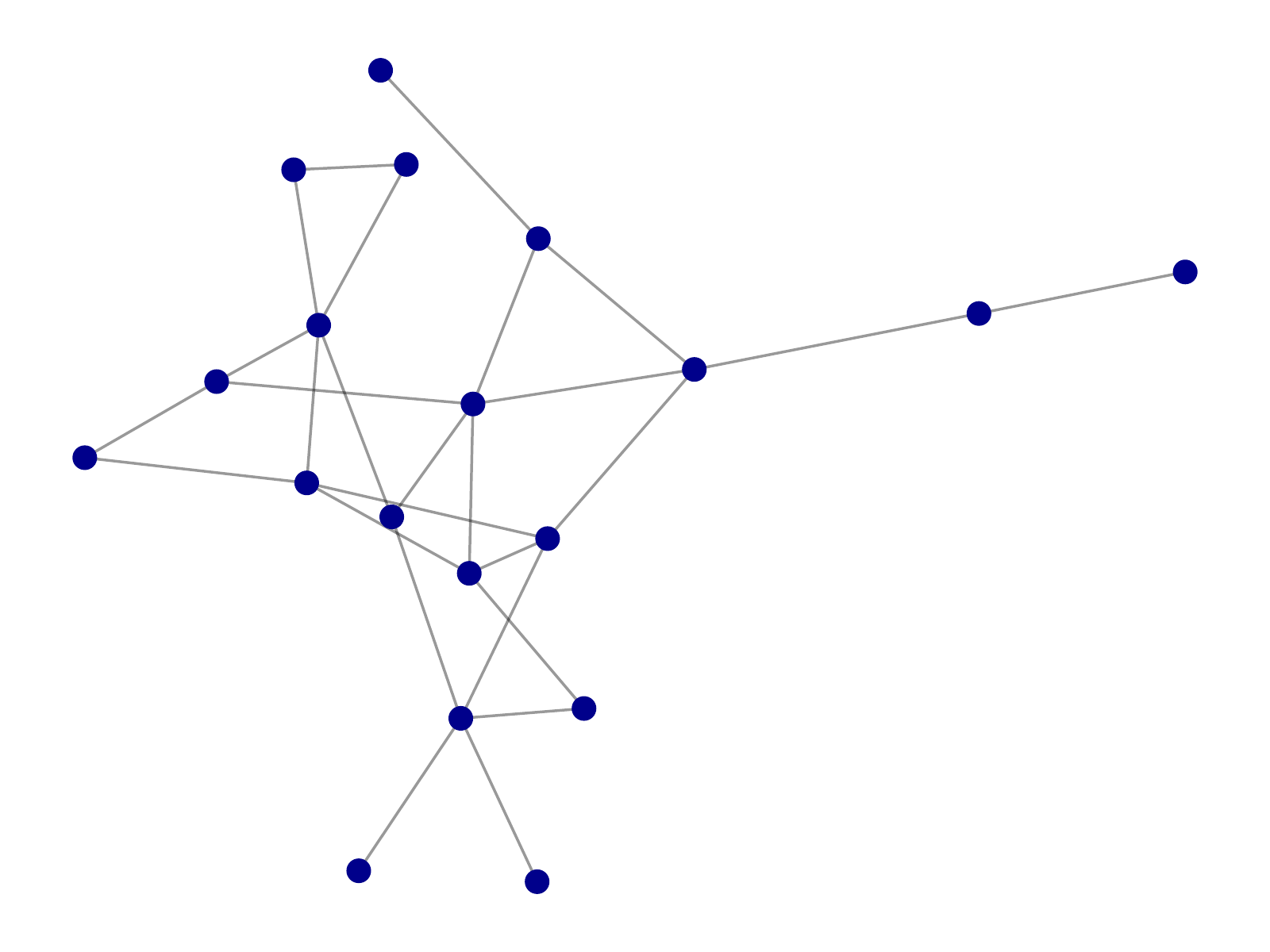}&\includegraphics[width=0.09\textwidth]{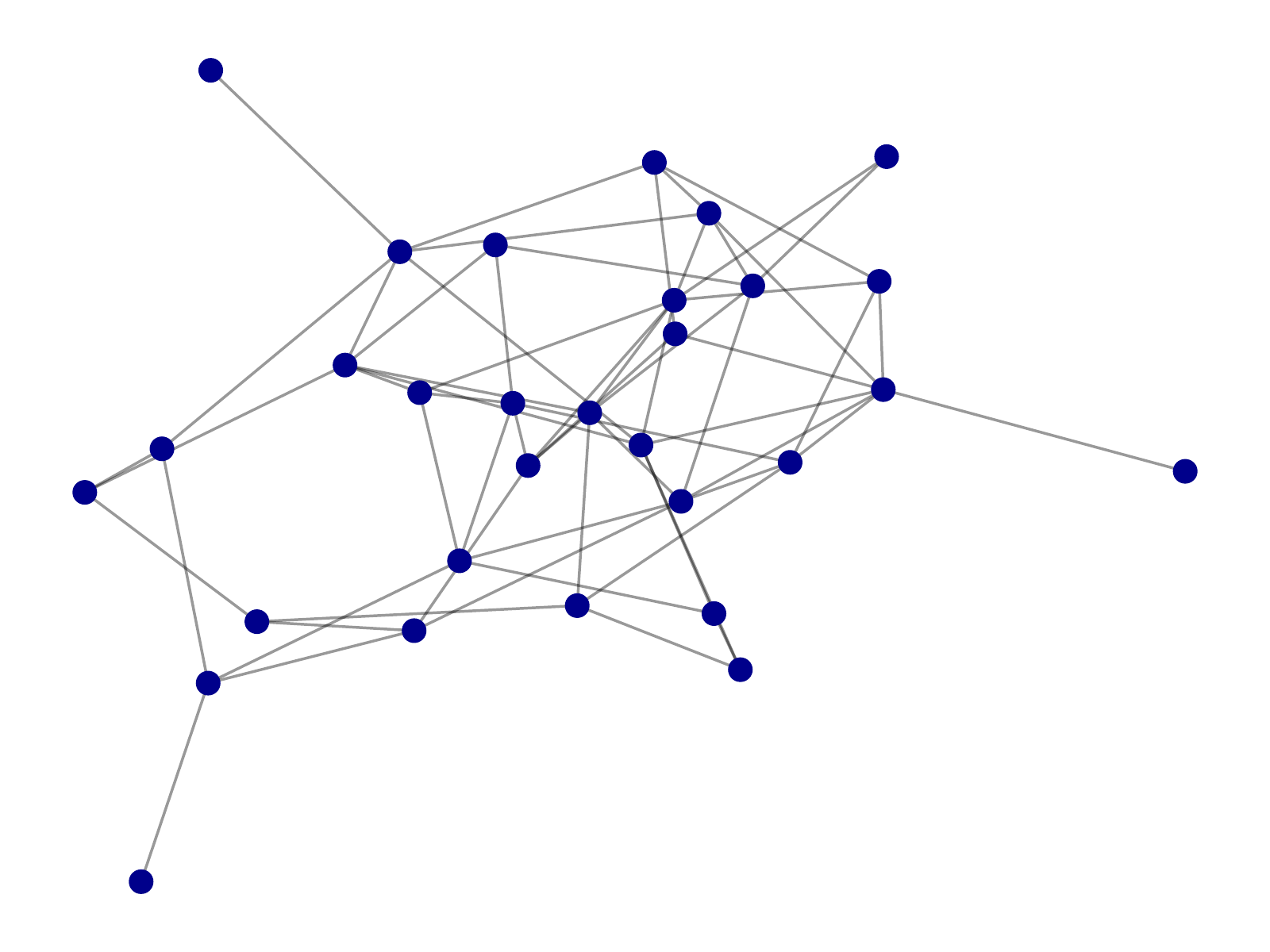}&
    \includegraphics[width=0.09\textwidth]{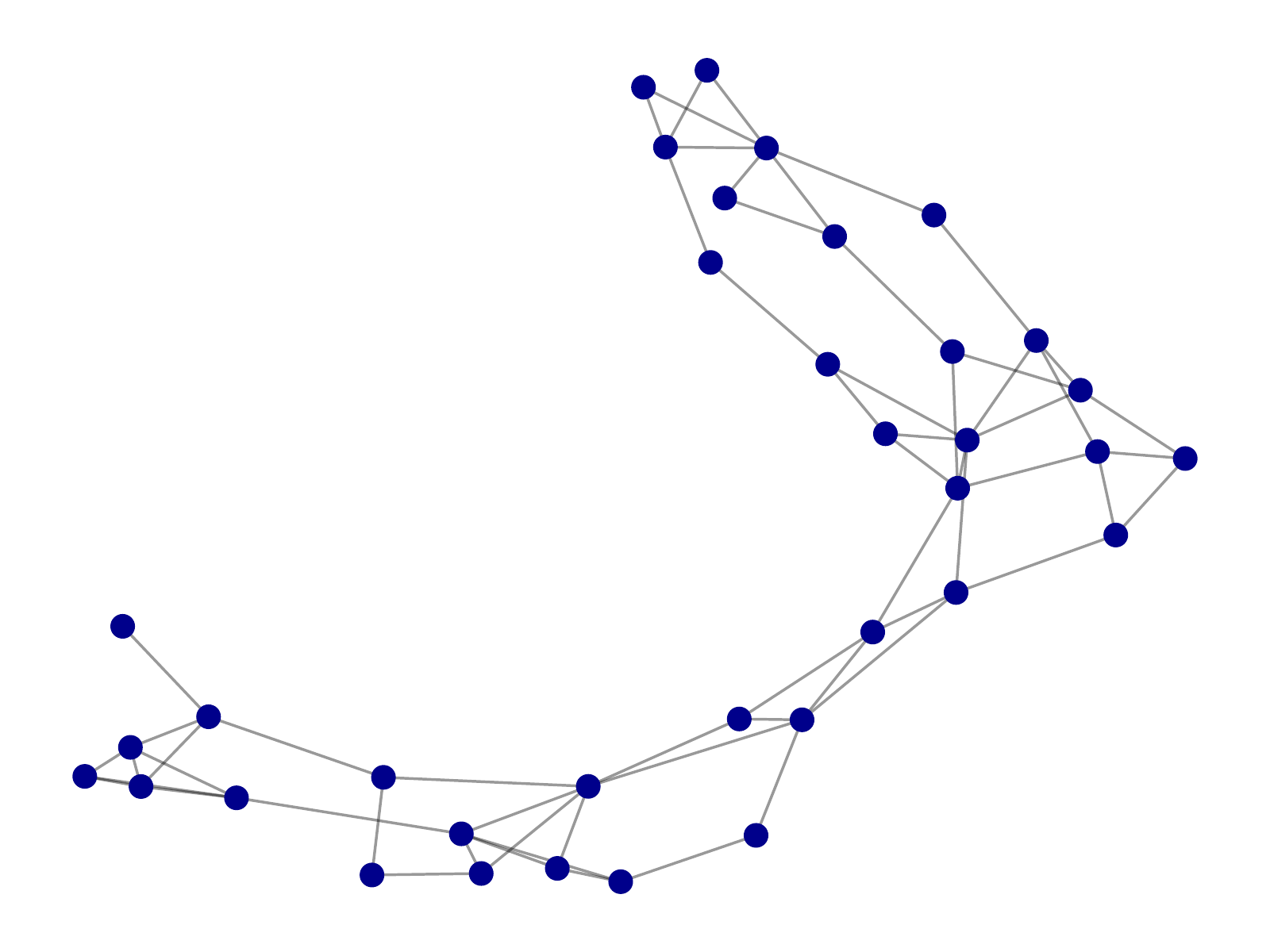}&\includegraphics[width=0.09\textwidth]{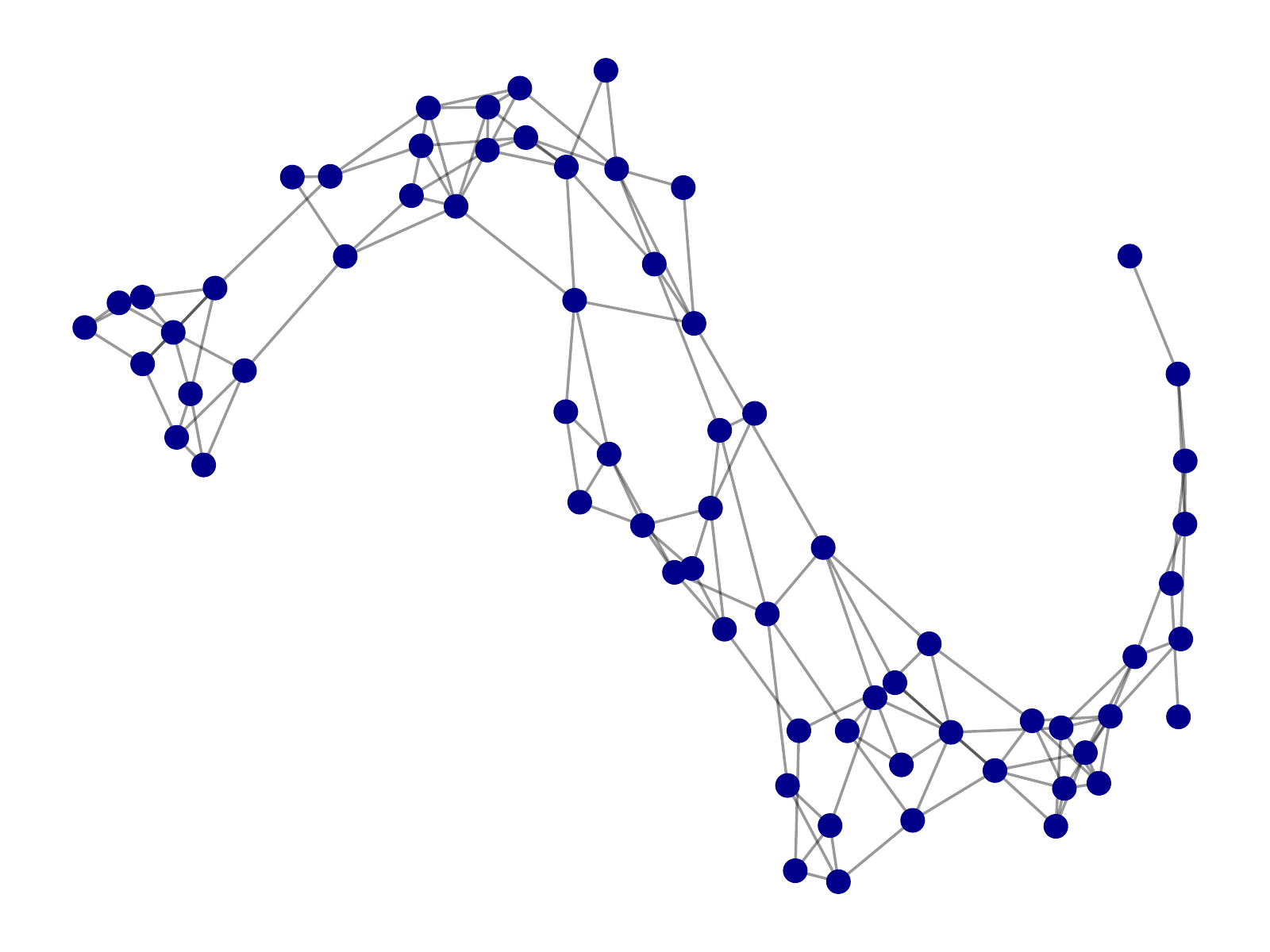}\\
    \multicolumn{8}{c}{Enzymes}
    \end{tabular}
    \includegraphics[width=\textwidth]{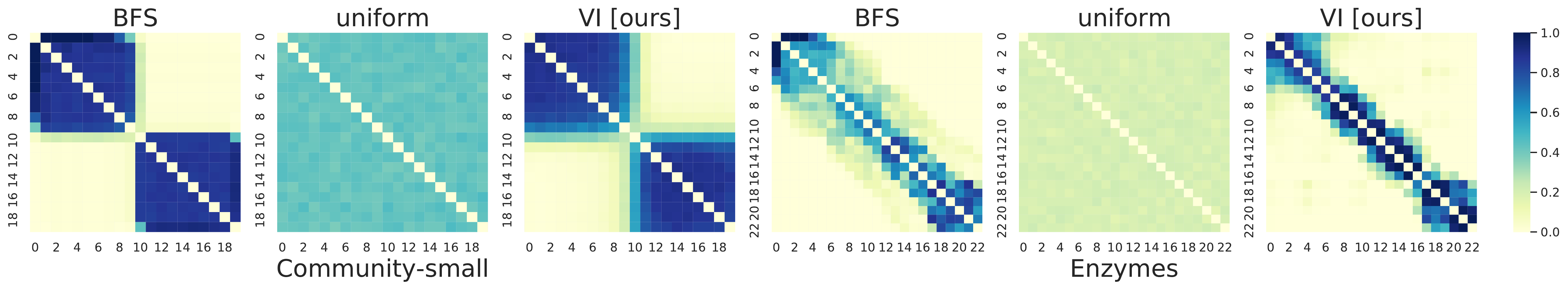}
    \vspace*{-20pt}
    \caption{(Top) Graph samples from different models trained on Community-small and Enzymes. The model fitted with VI learns to generate graphs with the same structural patterns as the real data. (Bottom) Averaged adjacency matrices for a graph with different samples from the node ordering. Our VI approach uncovers the underlying structure of the graph.}
    % \caption{Graph samples from different models trained on Community-small and adjacency matrices from different node orderings.}
    \label{fig:generation-comm}
    \vspace*{-4pt}
\end{figure*}

\begin{table*}[t]
    \small
    % \centering
    % \caption{Graph quality on labeled datasets}
    % \label{tab:label-graph-qualtiy}
    % \input{tables/labeled_graph_quality}
    % \small
    \centering
    \caption{Graph quality on the considered datasets (MMD on three metrics). Models fitted with VI tend to produce higher-quality samples.}
    % \caption{Graph quality on unlabeled datasets}
    \label{tab:unlabel-graph-qualtiy}
    \input{tables/unlabeled_graph_quality}
    \vspace*{-5pt}
\end{table*}

\section{Experiments}
\label{subsec:exp}

In this section, we design a set of experiments to investigate: (i) the tightness of the variational lower bound, (ii) the performance of a model fitted with the proposed method based on VI, (iii) the quality of the approximate posterior learned by the variational distribution, and (iv) the quality of graphs generated with the fitted model.

% This section gives a detailed description of our empirical study, which investigates: 1) the tightness of the variational lower bound, 2) the performance of data fitting in terms of MLE with the proposed method, 3) inference results from the variational distribution, and 4) the quality of graphs generated with the proposed method. 

\subsection{Experimental setup}

\parhead{Datasets.} We use 6 datasets: (1) \emph{Community-small}: 500 community graphs with $12\leq|V|\leq20$. Each graph has two communities generated by the model of \citet{erdHos1960evolution}. (2) \emph{Citeseer-small}: 200 subgraphs with $5\leq|V|\leq20$, extracted from Citeseer network \citep{sen2008collective} using random walk. (3) \emph{Enzymes}: 563 protein graphs from BRENDA database \citep{schomburg2004brenda} with $10\leq|V|\leq125$. (4) \emph{Lung}: 400 chemical graphs with $6\leq|V|\leq50$, sampled from \citet{10.1093/nar/gky1033}. (5) \emph{Yeast}: 400 chemical graphs with $5\leq|V|\leq50$, sampled from \citet{10.1093/nar/gky1033}. (6) \emph{Cora}: 400 subgraphs with $9\leq|V|\leq97$, extracted from the Cora network \citep{sen2008collective} using random walk.\looseness=-1

% \subsection{Experiment Setup}
% \textbf{Datasets:} We choose 6 real and synthetic datasets run all experiments: (1) Community-small: 500 small community graphs with $12\leq|V|\leq20$. Each graph has two communities generated by E-R model \cite{erdHos1960evolution}. (2) Citeseer-small: 200 subgraphs with $5\leq|V|\leq20$, extracted from Citeseer network \cite{sen2008collective} using random walk. (3) Enzymes: 563 protein graphs from BRENDA database \cite{schomburg2004brenda} with $10\leq|V|\leq125$. (4) Lung: 400 chemical graphs sampled from \cite{10.1093/nar/gky1033} with  $6\leq|V|\leq50$. (5) Yeast: 400 chemical graphs sampled from \cite{10.1093/nar/gky1033} with $5\leq|V|\leq50$. (6) Cora: 400 subgraphs with $9\leq|V|\leq97$, extracted from Cora network \cite{sen2008collective} using random walk. 
%In addition, we use a toy dataset that contains small Grid graphs with $4\leq|V|\leq49$, in order to demonstrate how our model learn a set of preferred orders.

\parhead{Methods.}
We choose three recent graph generative models, DeepGMG \cite{li2018learning}, GraphRNN \cite{you2018graphrnn}, and GraphGEN \cite{goyal2020graphgen}. We use their original training methods with default hyperparameters as baselines, and compare them with the proposed VI method. For our method, we use the Nauty package \citep{mckay2013nauty} to compute $|\Pi[A]|$ and the color refinement algorithm to approximate $|\Pi[G_{1:n}]|$, and we parameterize the variational distribution with a GAT \citep{velivckovic2017graph} with 3 layers, 6 attention heads, and residual connections.

% \textbf{Baselines:} We choose three recent graph generative models, DeepGMG \cite{li2018learning}, GraphRNN \cite{you2018graphrnn}, and \cite{goyal2020graphgen}. We use their original training methods as baselines and compare them with the proposed varitional inference method. In our method, we use the Nauty package for computing $|\Pi[A]|$. We use the color refinement algorithm approximate $|\Pi[G_{1:t}]|$. We train all models using default hyperparameters in their original papers. In our method, we use a GAT with 3 layers, 6 attention heads, and also residual connections.

\subsection{Predictive performance in terms of log-likelihood}

Here we compare the different methods in terms of the log-likelihood on test data.
We approximate the log-likelihood using importance sampling \citep{murphy2012machine}. We use the variational distribution $q_{\phi}(\pi | G)$ as the proposal distribution and draw $L$ samples $\{\pi^{(l)}\}$ from it. The importance sampling approximation of $\log p(G)$ is  
\begin{equation}
    \label{eq:approx-loglike}
    \log p_{\theta}(G) \simeq \log \Big(\frac{1}{L}\sum_{l=1}^{L} \frac{p_{\theta}(G,\pi_l^{(l)})}{q_{\phi}(\pi^{(l)}|G)}\Big).
\end{equation} 
Here $\pi^{(l)} \sim q_{\phi}(\pi|G)$ for $l = 1, \ldots, L$.
The estimation is unbiased only when $L$ approaches infinity; nevertheless, we found that $L=1{,}000$ gives an accurate estimation (see \Cref{app:est-lik}).

For our method, we use the learned $q_{\phi}(\pi|G)$ distribution as the proposal in the importance sampling approximation. For DeepGMG and GraphRNN, we use a uniform proposal $q_{\phi}(\pi|G)$,
because these methods are trained with node orderings sampled from the uniform distribution (as mentioned before, this is equivalent to using a uniform variational distribution). We use $L = 1{,}000$ samples for each graph in the test set, except for GraphGEN, for which we only use the canonical order $\pi^\star$ to estimate the log-likelihood.

%\todo{Fran: We may get criticized for this, because the resulting values are influenced not only by the model performance (what we are interested in) but also by this choice of $q$. It'd be nice to fit a $q$ distribution using our approach to a model that has been fitted using the baseline method, learning $\phi$ but keeping $\theta$ fixed. With that, we can see if/how much the resulting test log-lik for the basline methods improve.}

% \subsection{The log-likelihood of test data}

% In this section, we compared different models in terms of the approximate log-likelihood of the test data. We compute the approximate log-likelihood of the test set under a trained model by \ref{eq:approx-loglike}. The baseline models of DeepGMG and GraphRNN are trained with node orderings sampled from the uniform distribution. As mentioned before, this training method is equivalent to using the uniform distribution as the variational distribution in VI, so we can also calculate the ELBO for baseline models. We use the uniform distribution as the proposal in \ref{eq:approx-loglike} when we calculate the approximate log-likelihood. For our method, we use the learned $q$ distribution as the proposal. We use $L = 1000$ samples to estimate the log-likelihood of each data point.  We use the canonical order defined in GraphGEN to compute the log-likelihood with a single order. Note that the log-likelihood is not consistent with its sampling probability as the generation order is not guaranteed, but we include its results anyway.

The results are in \Cref{tab:marginal-likelihood}. We compare the results from each baseline and from our approach using a paired $t$-test at the $5\%$ significance level. We see that the proposed VI method exhibits better predictive performance on most datasets, and the improvements are often very significant. To assess the quality of the variational lower bound, we also show its value in \Cref{tab:marginal-likelihood} (the bound was estimated with $1000$ samples from $q_{\phi}(\pi|G)$). We can see that the bound is relatively tight for most cases. These results indicate that our training procedure based on VI can significantly improve the performance of a graph generative model.

On the Yeast dataset, the result of the VI approach is very close to the DeepGMG baseline. We checked the node orderings sampled from the learned variational distribution and observed that they are very similar to DFS orders. We hypothesize that, for this dataset, the posterior $p_{\theta}(\pi|G)$ is higher for DFS orders, and that $q_{\phi}(\pi|G)$ can find this. On the Community-small dataset, the gap with the baseline is much larger; this is because the graphs in this dataset have a special structure that always connects two communities with one edge. The variational distribution seems to be able to exploit this structure to improve the model fitting. For the Citeseer-small and Cora datasets, the gap is smaller---these datasets are generated from random walks, so the graphs have less structure for the VI algorithm to exploit.

\subsection{Qualitative analysis}

We now analyze qualitatively the graphs generated by each approach. Here we focus on GraphRNN models. \Cref{fig:generation-comm} (top left) shows four graphs from the Community-small dataset and four graphs from the Enzymes dataset. We then show graphs generated by variants of GraphRNNs that are trained with different node orderings (BFS, uniform, and our VI approach); these samples are representative and not cherry-picked. For Community-small, our method can capture the specific graph pattern---only one edge exists between two communities---with only one exception. The model trained with BFS orderings learns to generate two communities, but it does not generally use a single edge to connect them. The model fitted with uniform orderings fails to generate two communities.
These results can be explained by the plot of adjacency matrices in \Cref{fig:generation-comm} (bottom). In this figure, we choose one graph, sample node orderings from different distributions, and plot the average of their corresponding adjacency matrices. On Community-small, the BFS order produces an adjacency matrix whose two anti-diagonal blocks are near zero. We hypothesize that this pattern across all node orderings is easier for the model to learn.
The variational distribution discovers this pattern.

% We first focus on GraphRNN models trained with different methods. Figure \ref{fig:generation-comm} top shows eight graphs from two datasets, Community-small and Enzymes. We then show graphs generated by GraphRNNs trained with different orderings.  For Community-small, we can see that our method can capture the specific graph pattern -- only one edge exists between two communities. The model trained with BFS orderings learns to generate two communities, but sometimes it does not always use a single edge to connect the two coummunities.   And for uniform orderings, the model even fails to generate two communities. Note that these examples represent real samples are not cherry-picked. 

% The results above can be well explained by the plot of adjacency matrices. We choose one graph, sample node orderings from different distributions, and take the average of their corresponding adjacency matrices. These averages are plotted in Figure \ref{fig:generation-comm} bottom. In this generation task, a BFS order is good for generating two communities, as it has two off-diagonal clean blocks in its corresponding adjacency matrix, then the generative model needs little effort to learn these entries. The $q$ distribution can discover this pattern.    

We perform the same analysis for the Enzymes dataset. In \Cref{fig:generation-comm} (top), the samples from the VI training method are more similar to the ground truth data than for the baseline training methods---they have the shape of long strips, and two of them contains large cycles. \Cref{fig:generation-comm} (bottom) shows the averaged adjacency matrices; we can see that the variational distribution learns to form band matrices that have most non-zeros around the diagonal. In contrast, BFS orderings scatter non-zeros to a wider range. In \Cref{app:generation-sequences}, we provide a similar analysis for DeepGMG (which is based on graph sequences) on the Enzymes dataset.

\subsection{Quality of generated graphs}

Here we quantitatively assess the quality of generated graphs. Following previous works \citep{you2018graphrnn,liao2019efficient, goyal2020graphgen}, we measure the quality in terms of their similarity to a test set using different metrics: the degree distribution, clustering coefficients and occurrences of 4-node orbits. Then, we measure the difference between the test set and a set of generated graphs using the maximum mean discrepancy (MMD) between their respective distributions (lower MMD indicates a better model).

% \subsection{The quality of generated graphs}
% Following previous work, we also measure the quality the generated graphs in terms of their similarity to a test set with different metrics. Previous work \cite{you2018graphrnn,liao2019efficient, goyal2020graphgen} use the distributions of degree, clustering coefficients and occurrences of 4-node orbits to characterize a set of graphs. Then the difference between the test set and a set of generated graphs is measured by the Maximum Mean Discrepancy (MMD) between their respective distributions. When evaluating a generative model, the smaller is the MMD distance, the better is the model. 

\Cref{tab:unlabel-graph-qualtiy} shows the MMD evaluation on the six datasets. The VI training method improves the performance of the three models in four datasets (Community-small, Enzymes, Yeast, and Cora), with some minor performance drops on the other two datasets. On Citeseer-small, the VI method exhibits a performance drop on only one metric when it is applied on GraphRNN or GraphGEN; this is somewhat consistent with our previous results that the log-likelihood improvement on this dataset is less significant. Overall, the results indicate that an autoregressive generative model trained with VI produces higher-quality graphs.

%% file: tables/unlabeled_graph_quality.tex
\begin{tabular}{lc ccc c ccc c ccc}
    \hline
        & & \multicolumn{3}{c}{Community-small} && \multicolumn{3}{c}{Citeseer-small}&& \multicolumn{3}{c}{Enzymes}\\
        
        \cline{3-5}  \cline{7-9}\cline{11-13}
        & & Deg. &  Clus. & Orbit  && Deg. &  Clus. & Orbit && Deg. &  Clus. & Orbit\\
    \hline
        \multirow{2}{*}{DeepGMG}& uniform &0.2 & 0.978 & 0.40 &&  0.052&0.06 &\textbf{0.005} && 1.51& 0.95&0.29  \\
        & VI [ours] & \textbf{0.178}& \textbf{0.921}& \textbf{0.338} && \textbf{0.028}&\textbf{0.014} &\textbf{0.005} &&\textbf{1.01}& \textbf{0.48}&\textbf{0.27} \\
    \hline
        \multirow{3}{*}{GraphRNN}
        & BFS &0.034 &0.11 &0.009 &&  0.016&\textbf{0.05} &0.004 && 0.03&0.085&0.043\\
        & uniform &0.096 & 0.091 & 0.021 && \textbf{0.009}& 0.09 & 0.003  && 0.042& 0.104 & 0.074\\
        & VI [ours] & \textbf{0.018}& \textbf{0.01}& \textbf{0.008} && 0.08&\textbf{0.05} &\textbf{0.002} && \textbf{0.015}&\textbf{0.067}&\textbf{0.02}\\
    \hline
        \multirow{2}{*}{GraphGEN}& DFS &0.695 &0.931 &0.178 && 0.047 &\textbf{0.032} & 0.017 && 0.716 & 0.456 & 0.078\\
        & VI [ours] &\textbf{0.143} & \textbf{0.248}& \textbf{0.068} && \textbf{0.032} & 0.078& \textbf{0.008} && \textbf{0.346} & \textbf{0.440} &  \textbf{0.020}\\
    % \hline
% \end{tabular}
% \begin{tabular}{lclccclccclccc}
    \hline
        &&  \multicolumn{3}{c}{Lung} && \multicolumn{3}{c}{Yeast}&& \multicolumn{3}{c}{Cora}\\
    \cline{3-5}  \cline{7-9}\cline{11-13}
         && Deg. &  Clus. & Orbit && Deg. &  Clus. & Orbit && Deg. &  Clus. & Orbit \\
    \hline 
    \multirow{2}{*}{DeepGMG}& uniform &0.206 &\textbf{0.023} &0.224 && 0.547&0.242&0.470 &&\textbf{0.35} &0.27 & 0.11\\
        & VI [ours] & \textbf{0.189} & \textbf{0.023}&\textbf{0.2} && \textbf{0.324}&\textbf{0.118}&\textbf{0.258} && 0.36& \textbf{0.22}& \textbf{0.04}\\
    \hline
        \multirow{3}{*}{GraphRNN}
        & BFS & 0.103 & 0.301& 0.043 && 0.512 & 0.153 & 0.026  && 1.125 & 1.002 & 0.427\\
        & uniform  & 1.213 &  \textbf{0.002} &  0.081&& 0.746 & 0.351 & 0.070 &&0.188 & 0.206 & 0.200 \\
        & VI [ours] & \textbf{0.074} & 0.060 & \textbf{0.004} && \textbf{0.097} & \textbf{0.092} & \textbf{0.005}  && \textbf{0.066} & \textbf{0.171} & \textbf{0.052}\\
    \hline
        \multirow{2}{*}{GraphGEN}& DFS &0.049 &0.017 &\textbf{0.000} && 0.014&\textbf{0.003} &\textbf{0.000} &&0.099 &0.167 &0.122 \\
        & VI [ours] & \textbf{0.022} & \textbf{0.008}& \textbf{0.000} &&\textbf{0.012} &\textbf{0.003} &\textbf{0.000} && \textbf{0.056} &\textbf{0.103} & \textbf{0.069}\\
    \hline
    \end{tabular}

%% file: sections/conclusion.tex
\section{Conclusion}
\label{sec:conclusion}

In this paper, we analyze autoregressive graph generative models that are based on either the adjacency matrix or the graph sequence. We provide an in-depth discussion of the automorphism issue that raises when calculating the marginal likelihood of the graph. Using VI, we also address the intractable marginalization over node orderings for fitting a graph generative model. The experiment results show that the variational distribution learns reasonable orderings that improve the generative model's performance. Our variational lower bound is tighter than existing bounds on the marginal log-likelihood. We evaluate models based on their test log-likelihood and find that models fitted with our VI approach exhibit better predictive performance and are able to generate higher-quality graphs than previous methods.
The main limitation of our method is its scalability; thus it is not designed for large graphs. We expect future work will accelerate the algorithm to improve its scalability.
%Our method is not designed for large graphs, thus doesn't scale well. We encourage future works to accelerate the algorithm for better scalability.

% In this paper, we seek to define the graph generation as a data fitting problem. We clarify the graph automorphism problem during calculating marginal likelihood over $\pi$ under two types of generative models: one defining the probabilty for adjacency matrix, and another for graph sequence. Considering that the marginal likelihood for a given graph is intractable, we propose a variational bound which is much tighter than previous work. The recurrent order structure is able to explore  reasonable node orders for the generative model, and the proposed orders helps improving the performance compared with canonical orderings. Moreover, as previous metrics to be proved not reflecting the  generated graphs quality well, we first time calculating the true log-likelihood for graphs. We believe that the likelihood is the most significant metrics when figuring out a data fitting problem, and prove that the proposed ELBO is a good estimation for log-likelihood.

%% file: sections/acknowledgments.tex
\section*{Acknowledgements}
We thank Yujia Li for his insightful comments, and the anonymous reviewers for their constructive feedback.
The work was supported by NSF 1850358 and NSF 1908617. Xu Han was also supported by NSF 1934553.

%% file: sections/Appendix.tex
\section{Appendix}

\subsection{Variance of the Gradient Estimators}
\label{app:variance}

We use the score function estimator \citep{williams1992simple}  to obtain the gradients. In some applications, this estimator may suffer from high variance and make the training process unstable. Here, we study the variance of the score function estimator to make sure that it does not cause optimization issues in our application. To show the behavior of the optimizer, we plot the objective (the ELBO)  in \Cref{fig:score-function-variance}(right)  
and the variance of the gradient estimator in  \Cref{fig:score-function-variance}(left); both against training epochs. We can see that the objective decreases smoothly throughout optimization, indicating that the algorithm is stable. The three curves in the left plot show the variance of the gradients for different number of Monte Carlo samples; as expected, the variance decreases as the number of samples increases. Moreover, the variance from a relatively small sample size ($S=8$) is already decently low. This is because the variational distribution $q_{\phi}(\pi | G)$ tends to concentrate its probability mass to a small number of node orders, which can be seen from our analysis of the variational distribution (Figure~3 and \Cref{fig:generation-sequences-community}). Considering the tradeoff between computation time and variance, we set $S=8$ in all our experiments.
% \todo{fran: please increase the font size of Figure 4. It's hardly readable now.}

\begin{figure}[h]
    \centering

    \includegraphics[width=0.47\textwidth]{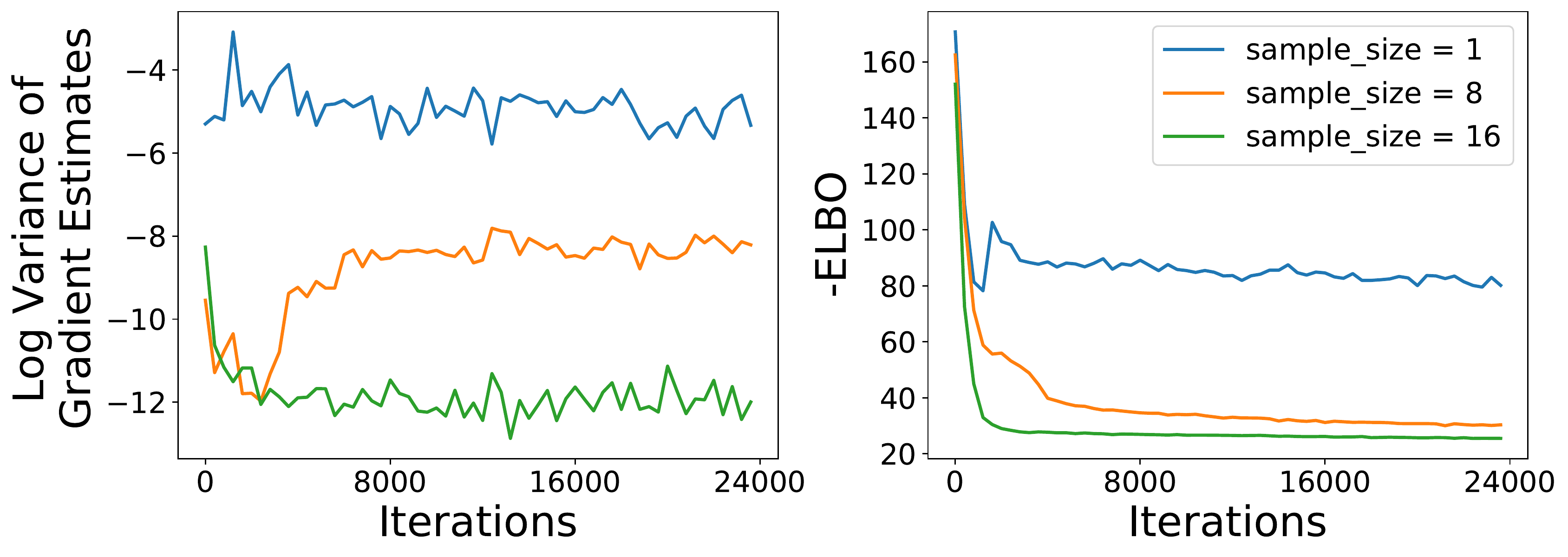}

    \caption{(Left) Training objective (ELBO) against epochs for GraphGEN on the Community-small dataset. The objective decreases smoothly throughout optimization. (Right) Log-variance of the score function gradient estimator for different number of Monte Carlo samples. Using $S=8$ samples is enough to estimate the gradient.}
    \label{fig:score-function-variance}
\end{figure}

\subsection{Proof of Lemma~1}

\setcounter{lemma}{0}
\begin{lemma}
    Let $G[V\backslash\{u\}]$ and  $G[V\backslash\{v\}]$ respectively denote the subgraphs induced by $V\backslash\{u\}$ and   $V\backslash\{v\}$, then $u$ and $v$ are in the same orbit if and only if $G[V\backslash\{u\}]$ and $G[V\backslash\{v\}]$ are isomorphic.
\end{lemma}

\begin{proof}

\label{app:proof_lemma1}
Let `$\equiv$' denote the isomorphic relation. Also denote $E\backslash \{u \} = \{(i, j) \in E: i \neq u, j\neq u\}$ as the subset of edges that do not incident $u$. 

We first show the first direction: ``$u$ and $v$ being in the same orbit'' indicates ``$G[V\backslash\{u\}] \equiv G[V\backslash\{v\}]$''. If $u$ and $v$ are in the same orbit, then $\exists f \in \text{Auto}(G) : f(v) = u$. Then $\forall i,j \in V\backslash\{u\}$, $(i, j) \in E\backslash\{u\} \Longleftrightarrow (f(i), f(j)) \in E\backslash\{v\}$ because $f$ is an automorphism. Then we restrict $f$ to $V\backslash\{u\}$ and get a injection $f': V\backslash\{u\} \rightarrow V\backslash\{v\}$, and $f'(i) = f(i) ~ \forall i \in V\backslash\{u\}$. Then $\forall i,j \in V\backslash\{u\}$, $(i, j) \in E\backslash\{u\} \Longleftrightarrow (f'(i), f'(j)) \in E\backslash\{v\}$. Therefore, $f'$ is an isomorphism  between $G[V\backslash\{u\}]$ and $G[V\backslash\{v\}]$.

We then prove by induction the second direction: ``$G[V\backslash\{u\}] \equiv G[V\backslash\{v\}]$'' indicates that ``$u$ and $v$ being in the same orbit'' .  

In the base case, we consider graphs with two nodes. Let $G$ be $(V=\{u,v\}, E=\emptyset)$ or $(V=\{u, v\}, E=(u,v))$. In either case, we always have $G \backslash \{u\} \equiv G \backslash \{u\}$. The two nodes $u$ and $v$ are also in the same orbit in both cases. So the second direction holds in the base case.

Then in the induction step, we assume the second direction is true for any graph of size $n$, then we show that it is also true for a graph of size $n+1$. Let $f \in \text{Auto}(G)$, and $f(u) = u'$. There are three cases: $u'$ is $v$, $u'$ is $u$, or $u'$ is neither of them. If it is the first case, then we have the conclusion directly: $u$ and $v$ are in the same orbit. 

Then we check the third case. With the same argument in the proof of the first direction, we restrict $f$ to $V\backslash \{u\}$ and get an isomorphism: $G \backslash \{u\} \equiv G \backslash \{u'\}$. By the condition $G \backslash \{u\} \equiv G \backslash \{v\}$, we also have $G \backslash \{u'\} \equiv G \backslash \{v\}$. We then remove $u$ from both graphs and get $G \backslash \{u, u'\} \equiv G \backslash \{u, v\}$. With the induction rule, we have that $u'$ and $v$ in the same orbit in $G \backslash \{u\}$. Let $g(\cdot) \in \text{Auto}(G \backslash \{u\})$ and $g(u') = v$. We extend $g(\cdot)$ to $V$ and let $g(u) = u$, then $g \circ f $ creates an automorphism on $G$, and $(g\circ f)(u) = v$. Therefore, $u$ is in the same orbit as $v$.

Finally, we show how to construct an $f(\cdot)$ such that $f(u) = u' \neq u$.  
Since $G[V\backslash \{v\}] \equiv G[V\backslash \{u\}]$, there is a isomorphism $h: V\backslash \{v\} \rightarrow  V\backslash \{u\}$, and $h(u) = u'$. Note that $u'$ cannot be $u$ because $u$ is not in the range of $h$. We extend $h$ to the domain $V$ and let $h(v) = u$, so $h$ is a permutation of $V$. For any $i, j \in E\backslash\{v\}$, $(i, j) \in E\backslash \{v\} \Longleftrightarrow (h(i), h(j)) \in E\backslash \{u\}$ because $h$ is an isomorphism. It is also true that $(i, j) \in E \Longleftrightarrow (h(i), h(j)) \in E$ because $(i, j)$ does not incident $v$, and $(f(i), f(j))$ does not incident $u$. Since $h$ is a permutation, the composition of $h$ forms a group: $\{h^0, h^{1}, \ldots, h^K\}$. The inverse $h^{-1}$ is the same as $h^{K}$. Let $j \in V \backslash \{v\}$, and $j = h^{-1}(i), ~ i \in V\backslash \{u\}$, then $(h(v), h(j)) = (u, i)$. With the previous argument, $(u, i) \in E \Longleftrightarrow (h(u), h(i)) \in E$. By the compisition rule, we further have $ (h(u), h(i)) \in E \Longleftrightarrow \ldots \Longleftrightarrow  (h^K(u), h^K(i)) = (v, j) \in E$. This works for any $j \in V\backslash {v}$, that is, $\forall j \in V \backslash \{v\}, ~ (v, j) \in E \Longrightarrow (h(v), h(j)) \in E$, then $h$ is a non-trivial automorphism on $G$ and $h(u) \neq u'$.
\end{proof}

\begin{figure*}[t]
    \centering
    \includegraphics[width=\textwidth]{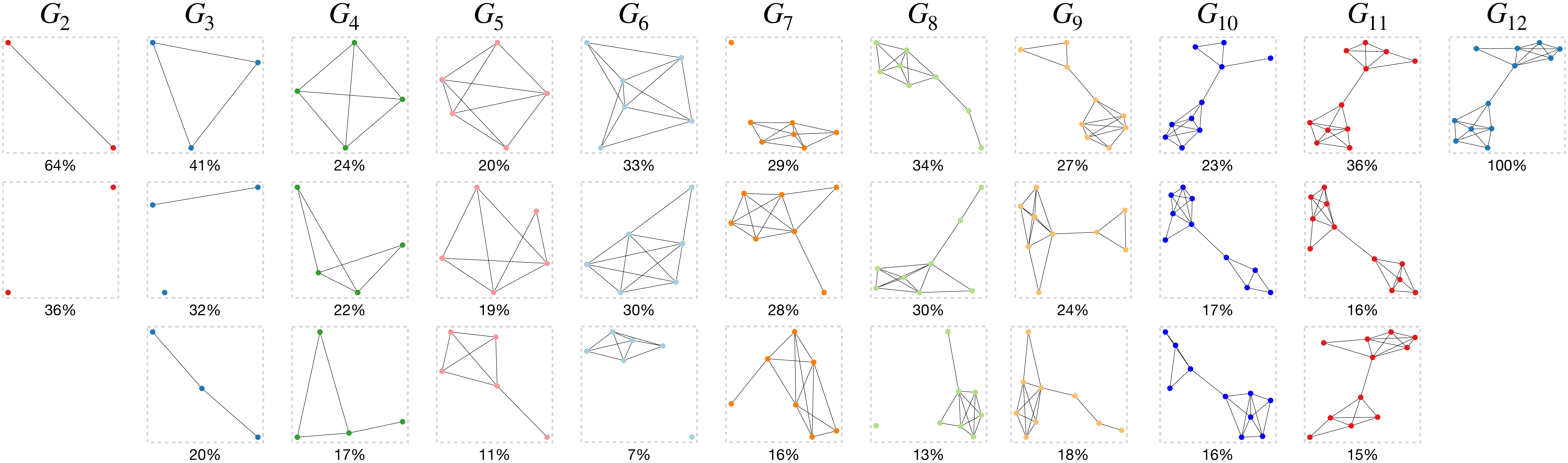}\\
    \caption{Graph generative sequences sampled from $q_{\phi}(\pi|G)$ for a graph from Community-small. The variational distribution prefers sequences of connected graphs. Similary to the distribution indicated in Figure \ref{fig:generation-comm} (bottom left), it first generates a comunity and then adds another one.}
    % \caption{graph generative sequences/adjacency matrix sampled from the variational distribution}
    \label{fig:generation-sequences-community}
\end{figure*}

\subsection{The Accuracy of the Log-Likelihood Estimation}
\label{app:est-lik}

To make sure we give an accurate estimation of the log-likelihood, we compare the estimated log-likelihood using different number of importance samples against the true log-likelihood. We compute the true log-likelihood of a graph by enumerating all possible permutations. We conduct the experiment on two datasets, Yeast and Lung. Since the calculation of the true log-likelihood is only feasible on small graphs, we keep graphs with fewer than $10$ nodes in each of the two datasets.  We use GraphRNN trained by variational inference as the model here, and the proposal distribution is the learned $q_{\phi}(\pi|G)$. \Cref{fig:estimated_vs_exact_loglik} shows the results on the two datasets. We see that when the number of samples is over $1000$, the gap between the true log-likelihood and the estimated log-likelihood becomes  very small (less than $0.1$). When we increase the number of samples to $2200$, the estimation is very accurate for both datasets. We conclude that the importance sampling estimator can be reliably used for model selection and model comparison.
% \todo{fran: suggestion: what about replacing ``sampled permutations'' in Fig 5 with ``importance samples''?}
\begin{figure}[t]
    \includegraphics[width=0.45\textwidth]{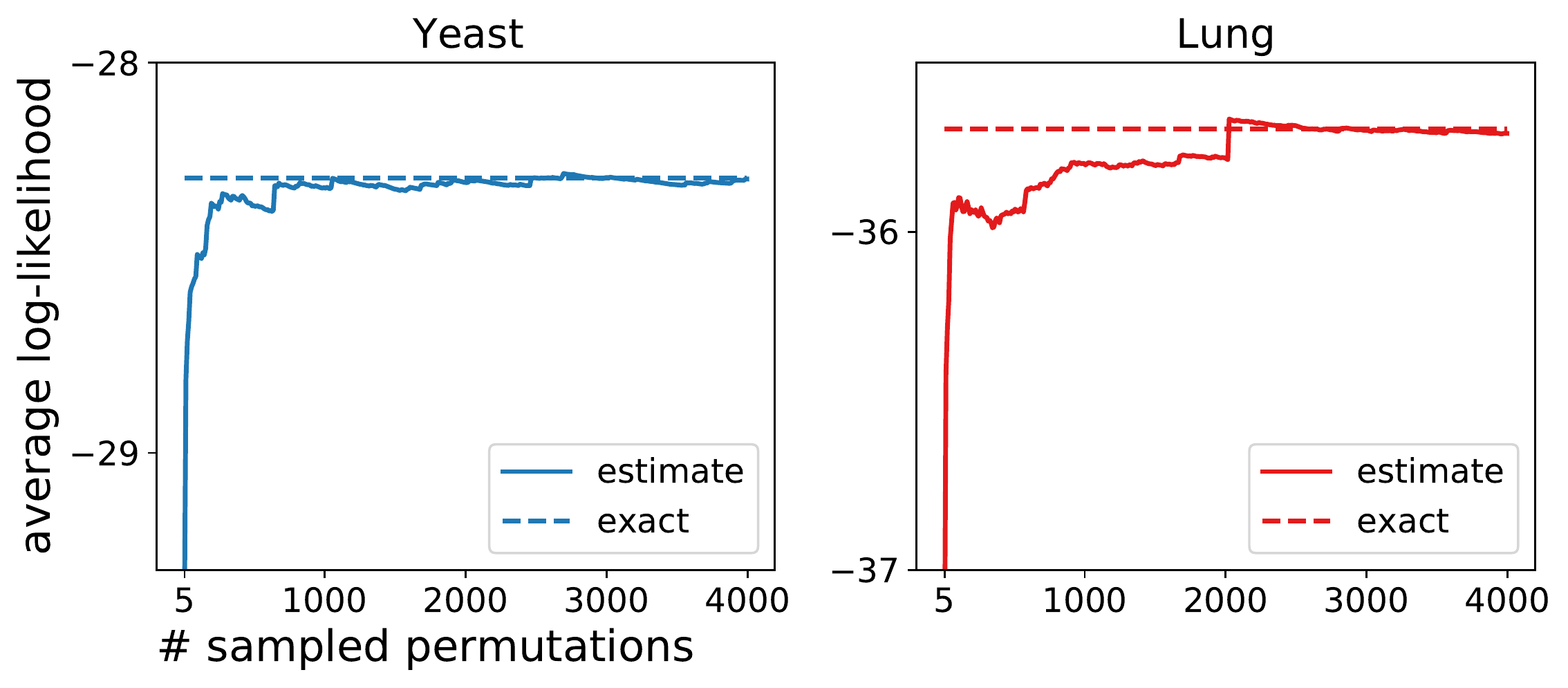}
    \caption{Comparison of estimated log-likelihood and exact log-likelihood for small graphs (fewer than 10 nodes) in the Lung and Yeast test sets. The estimation from $S=2200$ importance samples is very accurate.}
    \label{fig:estimated_vs_exact_loglik}
\end{figure}

\begin{figure*}[t]
    \centering
    \includegraphics[width=0.815\textwidth]{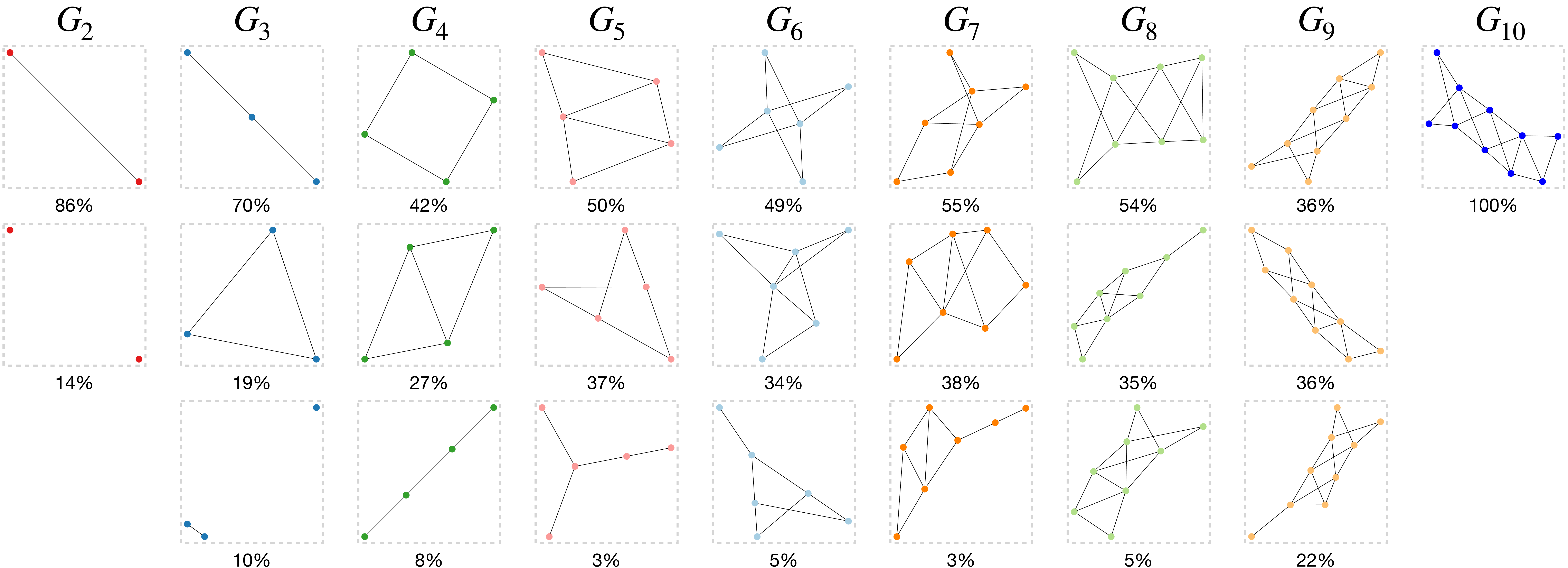}\\
    \caption{Graph generative sequences sampled from $q_{\phi}(\pi|G)$ for a graph from Enzymes. The variational distribution has a strong preference for sequences of connected graphs.}
    % \caption{graph generative sequences/adjacency matrix sampled from the variational distribution}
    \label{fig:generation-sequences-enzymes}
\end{figure*}

\newpage
\subsection{Graph Sequence Pattern in DeepGMG}
\label{app:generation-sequences}

In Section~4, we have investigated the variational distribution when training GraphRNNs. Here we study the variational distribution when training DeepGMG. For this experiment, we also consider the Community-small and the Enzymes dataset in order to show how our model learns a set of preferred orders. We choose the smallest graph from each dataset (a graph with $12$ nodes for Community-small and a graph with $10$ nodes for Enzymes). For each graph, we sample $720$ graph sequences from the trained variational distribution. We show the sampled graph sequences in \Cref{fig:generation-sequences-community} and \Cref{fig:generation-sequences-enzymes}. Without any prior knowledge, the variational distribution has strong preference for sequences of connected graphs. In addition, in Community-small, just like GraphRNN, the model prefers to generate communities one by one.